\documentclass{article}

\usepackage{microtype}
\usepackage{graphicx}
\usepackage{subcaption}
\usepackage{xcolor}
\usepackage{enumitem}
\usepackage{booktabs} %
\usepackage{tablefootnote}

\usepackage{hyperref}

\definecolor{candypink}{rgb}{0.89, 0.44, 0.48}          %
\definecolor{mediumaquamarine}{rgb}{0.4, 0.8, 0.67}     %
\definecolor{azure}{rgb}{0.0, 0.5, 1.0}                 %
\definecolor{awesome}{rgb}{1.0, 0.13, 0.32}             %

\newcommand{\csdlink}{\href{https://seohong.me/projects/csd/}{our project page}\xspace}
\newcommand{\csdvideo}{\href{https://seohong.me/projects/csd/}{videos}\xspace}
\newcommand{\csdaddress}{{\url{https://seohong.me/projects/csd/}}}

\usepackage[accepted]{icml2023}

\usepackage{amsmath}
\usepackage{amssymb}
\usepackage{mathtools}
\usepackage{amsthm}

\usepackage[capitalize,noabbrev]{cleveref}

\theoremstyle{plain}
\newtheorem{theorem}{Theorem}[section]

\newtheorem{lemma}[theorem]{Lemma}

\theoremstyle{definition}

\theoremstyle{remark}

\usepackage[textsize=tiny]{todonotes}

\usepackage{color}
\usepackage{amsmath}
\usepackage{amssymb}
\usepackage{multirow, varwidth}
\usepackage{dblfloatfix}
\usepackage{verbatim}
\usepackage{xspace}
\makeatletter
\newif\if@restonecol
\makeatother
\usepackage{xr}
\DeclareRobustCommand\onedot{\futurelet\@let@token\@onedot}
\def\onedot{.\xspace}
\def\eg{\emph{e.g}\onedot} 
\def\ie{\emph{i.e}\onedot}

\newcommand{\cutsectionup}{\vspace{-5pt}}
\newcommand{\cutsectiondown}{\vspace{-3pt}}
\newcommand{\cutsubsectionup}{\vspace{-3pt}}
\newcommand{\cutsubsectiondown}{\vspace{-3pt}}

\usepackage{amsmath,amsfonts,bm}

\def\eqref#1{equation~\ref{#1}}

\def\1{\bm{1}}

\DeclareMathAlphabet{\mathsfit}{\encodingdefault}{\sfdefault}{m}{sl}
\SetMathAlphabet{\mathsfit}{bold}{\encodingdefault}{\sfdefault}{bx}{n}

\def\gA{{\mathcal{A}}}

\def\gJ{{\mathcal{J}}}

\def\gM{{\mathcal{M}}}
\def\gN{{\mathcal{N}}}

\def\gP{{\mathcal{P}}}

\def\gS{{\mathcal{S}}}

\def\gZ{{\mathcal{Z}}}

\def\sR{{\mathbb{R}}}

\def\sV{{\mathbb{V}}}

\newcommand{\E}{\mathbb{E}}

\hypersetup{urlcolor=magenta}

\icmltitlerunning{Controllability-Aware Unsupervised Skill Discovery}

\begin{document}

\twocolumn[
\icmltitle{Controllability-Aware Unsupervised Skill Discovery}

\icmlsetsymbol{equal}{*}

\begin{icmlauthorlist}
\icmlauthor{Seohong Park}{berkeley}
\icmlauthor{Kimin Lee}{google}
\icmlauthor{Youngwoon Lee}{berkeley}
\icmlauthor{Pieter Abbeel}{berkeley}
\end{icmlauthorlist}

\icmlaffiliation{berkeley}{University of California, Berkeley}
\icmlaffiliation{google}{Google Research}

\icmlcorrespondingauthor{Seohong Park}{seohong@berkeley.edu}

\icmlkeywords{Machine Learning, ICML}

\vskip 0.3in
]

\printAffiliationsAndNotice{}  %

\begin{abstract}
One of the key capabilities of intelligent agents is the ability to discover useful skills without external supervision.
However, the current unsupervised skill discovery methods are often limited to
acquiring simple, easy-to-learn skills
due to the lack of incentives to discover more complex, challenging behaviors.
We introduce a novel unsupervised skill discovery method, \textbf{Controllability-aware Skill Discovery} (\textbf{CSD}),
which actively seeks complex, hard-to-control skills without supervision.
The key component of CSD is a controllability-aware distance function,
which assigns larger values to state transitions that are harder to achieve with the current skills. 
Combined with distance-maximizing skill discovery,
CSD progressively learns more challenging skills over the course of training
as our jointly trained distance function reduces rewards for easy-to-achieve skills.
Our experimental results in six robotic manipulation and locomotion environments demonstrate that
CSD can discover diverse complex skills
including object manipulation and locomotion skills with no supervision,
significantly outperforming prior unsupervised skill discovery methods.
Videos and code are available at \csdaddress
\end{abstract}

\vspace{-17pt}
\section{Introduction}
\cutsectiondown
\label{sec:intro}

Humans are capable of \emph{autonomously} learning skills, ranging from basic muscle control to complex acrobatic behaviors,
which can be later combined to achieve highly complex tasks.
Can machines similarly discover useful skills without any external supervision?
Recently, many unsupervised skill discovery methods have been proposed
to discover diverse behaviors in the absence of extrinsic rewards
\citep{vic_gregor2016,diayn_eysenbach2019,dads_sharma2020,valor_achiam2018,edl_campos2020,visr_hansen2020,ibol_kim2021,aps_liu2021,lsd_park2022,cic_laskin2022}.
These methods have also demonstrated efficient downstream reinforcement learning (RL) either by fine-tuning~\citep{urlb_laskin2021, cic_laskin2022} or sequentially combining~\citep{diayn_eysenbach2019, dads_sharma2020, lsd_park2022} the discovered skills.

\begin{figure}[t!]
    \centering
    \includegraphics[width=\linewidth]{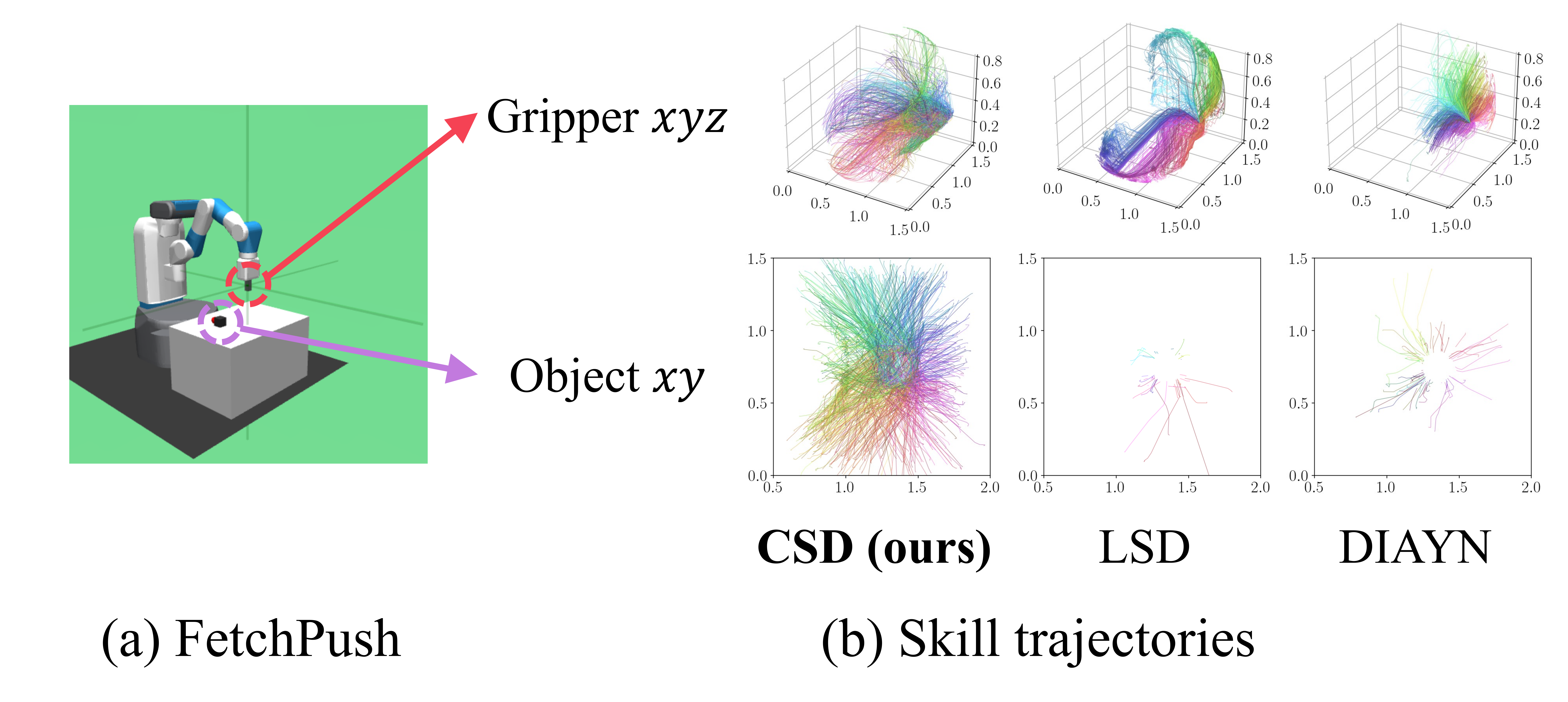}
    \vspace{-25pt}
    \caption{
    Object trajectories and gripper trajectories of 2-D continuous skills discovered
    by three unsupervised skill discovery methods, CSD (ours), LSD~\citep{lsd_park2022}, and DIAYN~\citep{diayn_eysenbach2019}, in the FetchPush environment.
    Trajectories with different colors represent different skills.
    While previous methods focus only on maneuvering the gripper,
    CSD discovers object manipulation skills in the absence of supervision.
    }
    \vspace{-10pt}
    \label{fig:teaser}
\end{figure}

However, in complex environments,
current unsupervised skill discovery methods are often limited to discovering only simple, easy-to-learn skills.
For example, as illustrated in \Cref{fig:teaser},
previous approaches (LSD and DIAYN) only learn to gain control of the agent's own `body' (\ie, the gripper and joint angles),
completely ignoring the object in the Fetch environment.
This is because
learning difficult skills, such as interacting with the object,
has no incentive for them compared to learning easy skills.
In other words, their objectives can be fully optimized with simple skills.

To mitigate this issue, prior approaches incorporate human supervision,
such as limiting the agent's focus to specific dimensions of the state space of interest
\citep{diayn_eysenbach2019,dads_sharma2020,lsd_park2022,irm_adeniji2022}.
However, this not only requires manual feature engineering but also significantly limits the diversity of skills.
On the other hand, we humans consistently challenge ourselves to learn more complex skills
after mastering simple skills in an autonomous manner.

Inspired by this, we propose a novel unsupervised skill discovery method,
\textbf{Controllability-aware Skill Discovery} (\textbf{CSD}),
which explicitly seeks complex, hard-to-learn behaviors that are potentially more useful for solving downstream tasks.
Our key idea is to train a controllability-aware distance function based on the current skill repertoire
and combine it with distance-maximizing skill discovery.
Specifically,
we train the controllability-aware distance function
to assign larger values to harder-to-achieve state transitions
and smaller values to easier-to-achieve transitions with the current skills.
Since CSD aims to maximize this controllability-aware distance,
it autonomously learns increasingly complex skills over the course of training.
We highlight that, to the best of our knowledge,
CSD is the first unsupervised skill discovery method that demonstrates diverse object manipulation skills in the Fetch environment
without any external supervision or manual feature engineering (\eg, limiting the focus only to the object).

To summarize, the main contribution of this work is
to propose CSD, a novel unsupervised skill discovery method built upon the notion of controllability.
We also formulate a general distance-maximizing skill discovery approach to incorporate our controllability-aware distance function with skill discovery.
We empirically demonstrate that
CSD discovers various complex behaviors,
such as object manipulation skills, with no supervision,
outperforming previous state-of-the-art skill discovery methods
in diverse robotic manipulation and locomotion environments.

\cutsectionup
\section{Preliminaries}
\cutsectiondown
\label{sec:preliminaries}

Unsupervised skill discovery aims at finding a potentially useful set of skills
without external rewards.
Formally,
we consider a reward-free Markov decision process (MDP) defined as $\gM = (\gS, \gA, \mu, p)$,
where $\gS$ and $\gA$ are the state and action spaces, respectively,
$\mu : \gP (\gS)$ is the initial state distribution,
and $p : \gS \times \gA \to \gP (\gS)$ is the transition dynamics function.
Each skill is defined as a skill latent vector $z \in \gZ$
and a skill-conditioned policy $\pi(a|s, z)$ that is shared across the skills.
The skill space $\gZ$ can be either discrete skills ($\{1, 2, \dots, D\}$) or continuous skills ($\sR^D)$.

To collect a skill trajectory (behavior),
we sample a skill $z$ from a predefined skill prior distribution $p(z)$ at the beginning of an episode.
We then roll out the skill policy $\pi(a|s, z)$ with the sampled $z$ for the entire episode.
For the skill prior, we use a standard normal distribution for continuous skills and a uniform distribution for discrete skills.

Throughout the paper, $I(\cdot; \cdot)$ denotes the mutual information and $H(\cdot)$ denotes either the Shannon entropy or differential entropy depending on the context.
We use uppercase letters for random variables and lowercase letters for their values
(\eg, $S$ denotes the random variable for states $s$).

\cutsectionup
\section{Related Work}
\cutsectiondown
\label{sec:related_work}

In this section, we mainly discuss closely related prior unsupervised skill discovery work
based on mutual information maximization or Euclidean distance maximization.
A more extensive literature survey on unsupervised skill discovery and unsupervised RL
can be found in \Cref{sec:appx_related}.

\cutsubsectionup
\subsection{Mutual Information-Based Skill Discovery}
\cutsubsectiondown
\label{sec:mi_related_work}

\begin{figure*}[t!]
    \centering
    \includegraphics[width=\textwidth]{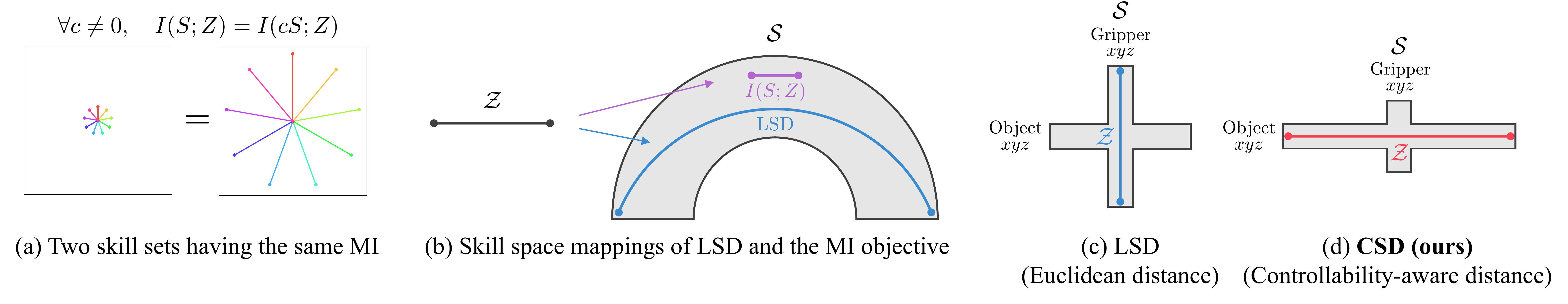}
    \vspace{-2em}
    \caption{
    Illustration of unsupervised skill discovery methods.
    (a) MI is invariant to traveled distances.
    (b) The MI objective simply seeks \emph{any} mapping between $\gZ$ and $\gS$,
    while LSD finds the largest (longest) possible mapping.
    (c) LSD maximizes the \emph{Euclidean} traveled distance, which can lead to simple or trivial behaviors.
    (d) Our CSD maximizes the traveled distance with respect to our learned \emph{controllability-aware} distance function
    that assigns larger values to harder-to-achieve state transitions.
    This leads to more complex skills that can be useful for downstream tasks.
    }
    \vspace{-10pt}
    \label{fig:mi_lsd_csd}
\end{figure*}

Mutual information-based unsupervised skill discovery maximizes
the mutual information (MI) between states $S$ and skills $Z$,
$I(S;Z)$,
which associates different states with different skill latent vectors
so that the behaviors from different $z$s are diverse and distinguishable.
Since computing exact MI is intractable, previous MI-based methods approximate MI in diverse ways,
which can be categorized into reverse-MI and forward-MI \citep{edl_campos2020}.

First, reverse-MI approaches~\citep{vic_gregor2016,diayn_eysenbach2019,valor_achiam2018,visr_hansen2020}
optimize MI in the form of $I(S; Z) = H(Z) - H(Z | S)$,
where $H(Z)$ is a constant as we assume that the skill prior distribution $p(z)$ is fixed.
Thus, maximizing $I(S; Z)$ corresponds to minimizing $H(Z|S)$,
which can be approximated with a variational distribution $q_\theta(z|s)$.
For instance, DIAYN \citep{diayn_eysenbach2019} maximizes the variational lower bound of MI as follows:
\begin{align}
    I(S; Z) &= -H(Z|S) + H(Z) \\
    &= \E_{z, s}[\log p(z|s)] - \E_{z}[\log p(z)] \\
    &\geq \E_{z, s}[\log q_\theta(z|s)] + (\text{const}), \label{eq:diayn}
\end{align}
where $q_\theta(z | s)$ is a variational approximation of $p(z | s)$ \citep{im_barber2003}.
Intuitively, $q_\theta(z|s)$ works as a `skill discriminator'
that tries to infer the original skill $z$ from the state $s$,
encouraging the skill policy to
generate distinguishable skill trajectories for different $z$s (\ie,~diverse skills).
Other reverse-MI methods %
optimize the MI objective similarly but computing MI on entire trajectories \citep{valor_achiam2018} or only on final states \citep{vic_gregor2016} rather than all intermediate states,
or using von Mises-Fisher distributions \citep{visr_hansen2020} for the skill prior distribution instead of Gaussian or uniform distributions.

On the other hand, forward-MI approaches~\citep{dads_sharma2020,edl_campos2020,aps_liu2021,cic_laskin2022}
employ the other decomposition of MI: $I(S; Z) = H(S) - H(S | Z)$.
This decomposition explicitly maximizes the state entropy $H(S)$,
which helps diversify skill trajectories in practice \citep{cic_laskin2022}.
Forward-MI methods minimize the $H(S | Z)$ term with
a variational approximation \citep{dads_sharma2020,aps_liu2021,edl_campos2020}
or a contrastive estimator \citep{cic_laskin2022}.
$H(S)$ can be estimated using
a particle-based entropy estimator \citep{aps_liu2021,cic_laskin2022},
a state marginal matching objective \citep{smm_lee2019,edl_campos2020},
or sampling-based approximation \citep{dads_sharma2020}.

One major limitation of MI-based approaches is that
optimizing the MI objective does not necessarily lead to covering a larger region in the state space.
This is because MI is invariant to traveled distances or any invertible transformation (\Cref{fig:mi_lsd_csd}a),
\ie, $I(S; Z) = I(f(S); Z)$ for any invertible $f$ \citep{mi_kraskov2004}.
Since there is no incentive for the MI objective to further explore the state space,
they often end up discovering `static' skills with limited state coverage~\citep{braxlines_gu2021,lsd_park2022,cic_laskin2022}.

\cutsubsectionup
\subsection{Euclidean Distance-Maximizing Skill Discovery}
\cutsubsectiondown
\label{sec:lsd}

To resolve the limitation of MI-based skill discovery, \citet{lsd_park2022} recently proposed Lipschitz-constrained Skill Discovery (LSD),
which aims to not only establish a mapping between $Z$ and $S$
but also maximize the Euclidean traveled distance in the state space for each skill.
Specifically,
LSD maximizes the state change along the direction specified by the skill $z$
with the following objective:
\begin{align}
    \gJ^{\text{LSD}} &:= \E_{z, s, s'}[(\phi(s') - \phi(s))^\top z] \label{eq:lsd_obj} \\
    & \text{s.t.} \quad \forall x, y \in \gS, \quad \| \phi(x) - \phi(y) \| \leq \| x - y \|, \label{eq:lsd_cst}
\end{align}
where $s'$ denotes the next state and $\phi: \gS \to \sR^D$ denotes a mapping function.
LSD maximizes \Cref{eq:lsd_obj} with respect to both the policy and $\phi$.
Intuitively, this objective aims to align the directions of $z$ and $(\phi(s') - \phi(s))$
while maximizing the length $\|\phi(s') - \phi(s)\|$, which leads to an increase in the state difference $\|s' - s\|$ due to the Lipschitz constraint.
As illustrated in \Cref{fig:mi_lsd_csd}b,
LSD finds the largest possible mapping in the state space
by maximizing Euclidean traveled distances in the state space in diverse directions,
which leads to more `dynamic' skills.
On the other hand, the MI objective finds \emph{any} mapping between the skill space and the state space,
being agnostic to the area of the mapped region,
which often results in `static' skills with limited state coverage.

While promising,
LSD is still limited in that it maximizes \emph{Euclidean} traveled distances in the state space,
which
often does not match the behaviors of our interests
because the Euclidean distance treats all state dimensions equally.
For example, in the Fetch environment in \Cref{fig:teaser},
simply diversifying the position and joint angles of the robot arm is sufficient to achieve large Euclidean traveled distances
because both the coordinates of the object and the gripper lie in the same Euclidean space (\Cref{fig:mi_lsd_csd}c).
As such, LSD and any previous MI-based approaches mostly end up learning skills
that only diversify the agent's own internal states, ignoring the external states (\eg, object pose).

Instead of maximizing the Euclidean distance, we propose to maximize traveled distances with respect to a learned \emph{controllability-aware distance function}
that `stretches' the axes along hard-to-control states (\eg, objects)
and `contracts' the axes along easy-to-control states (\eg, joint angles),
so that maximizing traveled distances results in the discovery of more complex, useful behaviors (\Cref{fig:mi_lsd_csd}d).

\cutsubsectionup
\subsection{Unsupervised Goal-Conditioned RL}
\cutsubsectiondown

Another line of unsupervised RL focuses on
discovering a wide range of \emph{goals} and learning corresponding goal-reaching policies,
which leads to diverse learned behaviors
\citep{discern_wardefarley2019,skewfit_pong2020,mega_pitis2020,lexa_mendonca2021}.
On the other hand, unsupervised skill discovery, including our approach,
(1) focuses on more general behaviors (\eg, running, flipping) not limited to goal-reaching skills,
whose behaviors tend to be `static' \citep{lexa_mendonca2021,rest_jiang2022},
and (2) aims to learn a \emph{compact} set of distinguishable skills
embedded in a low-dimensional, possibly discrete skill space,
rather than finding all possible states,
making it more amenable to hierarchical RL
by providing a low-dimensional high-level action space (\ie, skill space).
While these two lines of approaches are not directly comparable,
we provide empirical comparisons and further discussion in \Cref{sec:appx_disag}.

\cutsectionup
\section{Controllability-Aware Skill Discovery}
\cutsectiondown
\label{sec:method}

To discover complex, useful skills without extrinsic reward and domain knowledge,
we introduce the notion of \emph{controllability}\footnote{
The term \textit{controllability} in this paper describes whether an agent can manipulate hard-to-control states (\eg, external objects) or not, different from the one used in control theory~\citep{ct_ogata2010}.
} to skill discovery --
once an agent discovers easy-to-achieve skills,
it continuously moves its focus to hard-to-control states and learns more diverse and complex skills.
We implement this idea in our Controllability-aware Skill Discovery (CSD) by combining a distance-maximizing skill discovery approach (\Cref{sec:dsd})
with a \emph{jointly} trained controllability-aware distance function (\Cref{sec:learned_distance}),
which enables the agent to find increasingly complex skills over the course of training (\Cref{sec:csd}).

\cutsubsectionup
\subsection{General Distance-Maximizing Skill Discovery}
\cutsubsectiondown
\label{sec:dsd}

As explained in \Cref{sec:lsd},
Euclidean distance-maximizing skill discovery does not necessarily maximize distances along hard-to-control states (\ie, hard-to-achieve skills).
To discover more challenging skills, we propose to learn a skill policy with respect to a jointly learned controllability-aware distance function.

To this end, we first present a general \textbf{Distance-maximizing Skill Discovery} approach (\textbf{DSD})
that can be combined with any arbitrary distance function $d(\cdot, \cdot): \gS \times \gS \to \sR_0^+$.
Specifically, we generalize the Euclidean distance-maximizing skill discovery \citep{lsd_park2022} by replacing $\| x - y \|$ in \Cref{eq:lsd_cst} with $d(x, y)$
as follows:
\begin{align}
    \gJ^\text{DSD} &:= \E_{z, s, s'}[(\phi(s') - \phi(s))^\top z] \label{eq:dsd_obj} \\
    & \text{s.t.} \quad \forall x, y \in \gS, \quad \| \phi(x) - \phi(y) \| \leq d(x, y), \label{eq:dsd_cst}
\end{align}
where
$\phi(\cdot): \gS \to \sR^D$ is
a function that maps states into a $D$-dimensional space
(which has the same dimensionality as the skill space).
DSD can discover skills that maximize the traveled distance under the given distance function $d$ in diverse directions
by (1)~aligning the directions of $z$ and $(\phi(s') - \phi(s))$ and (2)~maximizing its length $\|\phi(s') - \phi(s)\|$,
which also increases $d(s, s')$ due to the constraint in \Cref{eq:dsd_cst}.
Here, LSD can be viewed as a special case of DSD with $d(x, y) = \|x - y\|$.

When dealing with a \emph{learned} distance function $d$,
it is generally not straightforward to ensure that $d$ is a \emph{valid} distance (pseudo-)metric,
which must satisfy symmetry and the triangle inequality.
However, DSD has the nice property that $d$ in \Cref{eq:dsd_cst} does not have to be a valid metric.
This is because
DSD implicitly converts the original constraint (\Cref{eq:dsd_cst}) into the one
with a valid pseudometric $\tilde{d}$. %
As a result, we can use any arbitrary non-negative function $d$ for DSD,
with the semantics being implicitly defined by its \emph{induced pseudometric} $\tilde{d}$. 
We summarize our theoretical results as follows and the proofs are in \Cref{sec:appx_proof}.

\begin{theorem}
\label{thm:dsd}
Given any non-negative function $d: \gS \times \gS \to \sR^+_0$, there exists a valid pseudometric $\tilde{d}: \gS \times \gS \to \sR^+_0$ that satisfies the following properties:
\begin{enumerate}
\item Imposing \Cref{eq:dsd_cst} with $d$ is equivalent to imposing \Cref{eq:dsd_cst} with $\tilde{d}$, \ie,
\begin{align}
&\forall x, y \in \gS, \quad \|\phi(x)-\phi(y)\| \leq d(x, y) \\
\iff &\forall x, y \in \gS, \quad \|\phi(x)-\phi(y)\| \leq \tilde{d}(x, y).
\end{align}
\item $\tilde{d}$ is a valid pseudometric.
\item $\tilde{d}$ is a lower bound of $d$, \ie,
\begin{align}
\forall x, y \in \gS, \quad 0 \leq \tilde{d}(x, y) \leq d(x, y).
\end{align}
\end{enumerate}
\end{theorem}

\textbf{Training of DSD.}
While LSD implements the Lipshitz constraint in \Cref{eq:lsd_cst} using Spectral Normalization~\citep{sn_miyato2018},
similarly imposing DSD's constraint in \Cref{eq:dsd_cst} is not straightforward because it is no longer a Euclidean Lipschitz constraint.
Hence, we optimize our objective with dual gradient descent \citep{convex_boyd2004}:
\ie, with a Lagrange multiplier $\lambda \geq 0$,
we use the following dual objectives to train DSD:
\begin{align}
    r^\text{DSD} &:= (\phi(s') - \phi(s))^\top z, \label{eq:dsd1} \\
    \gJ^{\text{DSD},\phi} &:= \E [ (\phi(s') - \phi(s))^\top z \nonumber \\
     & \quad + \lambda \cdot \min (\epsilon, d(x, y) - \|\phi(x) - \phi(y)\|) ], \label{eq:dsd2} \\
    \gJ^{\text{DSD},\lambda} &:= -\lambda \cdot \E [\min (\epsilon, d(x, y) - \|\phi(x) - \phi(y)\|) ], \label{eq:dsd3}
\end{align}
where $r^\text{DSD}$ is the intrinsic reward for the policy,
and $\gJ^{\text{DSD},\phi}$ and $\gJ^{\text{DSD},\lambda}$
are the objectives for $\phi$ and $\lambda$, respectively.
$x$ and $y$ are sampled from some state pair distribution $p^{\text{cst}}(x, y)$
that imposes the constraint in \Cref{eq:dsd_cst}.
$\epsilon > 0$ is a slack variable
to avoid the gradient of $\lambda$ always being non-negative.
With these objectives, we can train DSD by optimizing the policy with \Cref{eq:dsd1} as an intrinsic reward
while updating the other components with \Cref{eq:dsd2,eq:dsd3}.

\cutsubsectionup
\subsection{Controllability-Aware Distance Function}
\cutsubsectiondown
\label{sec:learned_distance}

To guide distance-maximizing skill discovery to focus on more challenging skills,
a distance function $d$ is required to assign larger values to state transitions that are hard-to-achieve with the current skills
and smaller values to easy-to-achieve transitions.
$d$ also needs to be adaptable to the current skill policy
so that the agent continuously acquires new skills
and finds increasingly difficult state transitions over the course of training. 

Among many potential distance functions,
we choose a negative log-likelihood of a transition from the current skill policy,
$-\log p(s'|s)$, as a \emph{controllability-aware distance function} in this paper.
Accordingly, we define the degree to which a transition is ``hard-to-achieve'' as $-\log p(s'|s)$
with respect to the current skill policy's transition distribution.
This suits our desiderata since
(1)~it assigns high values for rare transitions (\ie, low $p(s'|s)$)
while assigns small values for frequently visited transitions (\ie, high $p(s'|s)$);
(2)~$p(s'|s)$ can be approximated by training a density model $q_\theta(s'|s)$ from policy rollouts;
and (3)~the density model $q_\theta(s'|s)$ continuously adjusts to the current skill policy
by jointly training it with the skill policy.
Here, while it is also possible to employ multi-step transitions $p(s_{t+k}|s_t)$
for the distance function,
we stick to the single-step version for simplicity.
We note that even though we employ single-step log-likelihoods,
DSD maximizes the sum of rewards, $\sum_{t=0}^{T-1} (\phi(s_{t+1}) - \phi(s_t))^\top z = (\phi(s_T) - \phi(s_0))^\top z$
for the trajectory $(s_0, a_0, s_1, \dots, s_T)$,
which maximizes the traveled distance of the \emph{whole} trajectory while maintaining the directional alignment with $z$.

\cutsubsectionup
\subsection{Controllability-Aware Skill Discovery}
\cutsubsectiondown
\label{sec:csd}

Now, we introduce \textbf{Controllability-aware Skill Discovery} (\textbf{CSD}),
a distance-maximizing skill discovery method with our controllability-aware distance function. %
With the distance function in \Cref{sec:learned_distance}
we can rewrite the constraint of DSD in \Cref{eq:dsd_cst} as follows:
\begin{align}
    \forall s, s' \in \gS, &\quad \| \phi(s) - \phi(s') \| \leq d^{\text{CSD}}(s, s'), \\
    d^{\text{CSD}}(s, s') &\triangleq (s' - \mu_\theta(s))^\top \Sigma_\theta^{-1}(s) (s' - \mu_\theta(s)) \\
    &\propto - \log q_\theta(s'|s) + (\text{const}), \label{eq:csd_dist}
\end{align}
where the density model is parameterized as $q_\theta(s'|s) = \gN(\mu_\theta(s), \Sigma_\theta(s))$,
which is jointly trained using $(s, s')$ tuples collected by the skill policy.
We also use the same $p(s, s')$ distribution from the skill policy
for the dual constraint distribution $p^{\text{cst}}(x, y)$ introduced in \Cref{sec:dsd} as well.
Here, we note that $d^{\text{CSD}}(\cdot, \cdot)$ is not necessarily a valid distance metric;
however, we can still use it for the constraint in \Cref{eq:dsd_cst} according to \Cref{thm:dsd},
because it automatically transforms $d^{\text{CSD}}$ into its induced valid pseudometric $\tilde{d}^{\text{CSD}}$.
Further discussion about its implications and limitations can be found in \Cref{sec:appx_implication}.

CSD has several main advantages.
First, the agent actively seeks rare state transitions and thus acquires increasingly complex skills
over the course of training,
which makes the skills discovered more useful for downstream tasks.
In contrast, LSD or previous MI-based approaches
only maximize Euclidean distances or are even agnostic to traveled distances,
which often leads to simple or static behaviors.
Second,
unlike LSD, the optimal behaviors of CSD are agnostic to the semantics and scales of each dimension of the state space; thus, CSD does not require domain knowledge about the state space.
Instead, the objective of CSD only depends on the difficulty or sparsity of state transitions.
Finally, unlike curiosity- or disagreement-based exploration methods
that only seek unseen transitions \citep{icm_pathak2017,disag_pathak2019,lexa_mendonca2021},
CSD finds a balance between covering unseen transitions 
and learning maximally different skills across $z$s via directional alignments,
which leads to diverse yet consistent skills.

\newlength{\textfloatsepsave} \setlength{\textfloatsepsave}{\textfloatsep} \setlength{\textfloatsep}{0pt}
\begin{algorithm}[t!]
    \caption{Controllability-aware Skill Discovery (CSD)}
    \begin{algorithmic}[1]
        \STATE Initialize skill policy $\pi(a|s, z)$, function $\phi(s)$,
        conditional density model $q_\theta(s'|s)$, Lagrange multiplier $\lambda$
        \FOR{$i \gets 1$ to (\# epochs)}
            \FOR{$j \gets 1$ to (\# episodes per epoch)}
                \STATE Sample skill $z \sim p(z)$
                \STATE Sample trajectory $\tau$ with $\pi(a|s, z)$
            \ENDFOR
            \STATE Fit conditional density model $q_\theta(s'|s)$ using current trajectory samples
            \STATE Update $\phi(s)$ with gradient ascent on $\gJ^{\text{DSD},\phi}$ %
            \STATE Update $\lambda$ with gradient ascent on $\gJ^{\text{DSD},\lambda}$ %
            \STATE \mbox{Update $\pi(a|s, z)$ using SAC with intrinsic reward $r^{\text{DSD}}$} %
        \ENDFOR
    \end{algorithmic}
    \label{alg:csd}
\end{algorithm}
\setlength{\textfloatsep}{\textfloatsepsave}
\begin{figure}[t!]
    \centering
    \includegraphics[width=\linewidth]{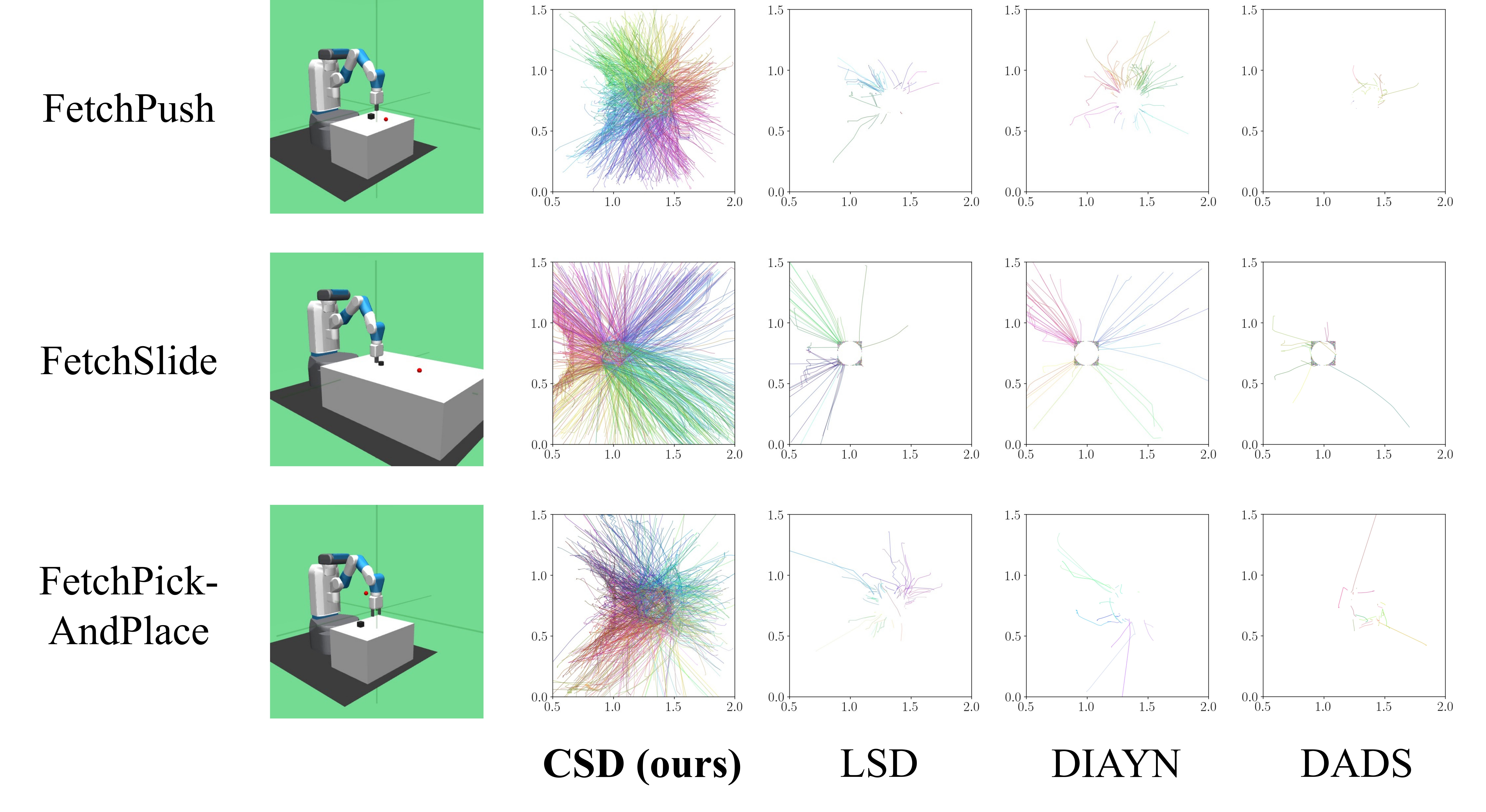}
    \vspace{-2em}
    \caption{
    The object trajectories in the $xy$ plane of randomly sampled $1000$ continuous skills
    learned by CSD, LSD, DIAYN, and DADS in three Fetch manipulation environments without any supervision.
    Trajectories with different colors represent different skills.
    Only CSD learns to manipulate the object across all three tasks without supervision
    while other methods focus only on moving the robot arm.
    We refer to \Cref{sec:appx_add_results} for the complete qualitative results
    from all random seeds. %
    }
    \vspace{-15pt}
    \label{fig:qual_fetch}
\end{figure}

\textbf{Training of CSD.}
We train the skill policy $\pi(a|s, z)$ with Soft Actor-Critic (SAC) \citep{saces_haarnoja2018} with
\Cref{eq:dsd1} as an intrinsic reward.
We train the other components with stochastic gradient descent.
We summarize the training procedure of CSD in \Cref{alg:csd}
and provide the full implementation details in \Cref{sec:appx_impl_details}.

\begin{figure*}[t!]
    \centering
    \includegraphics[width=\linewidth]{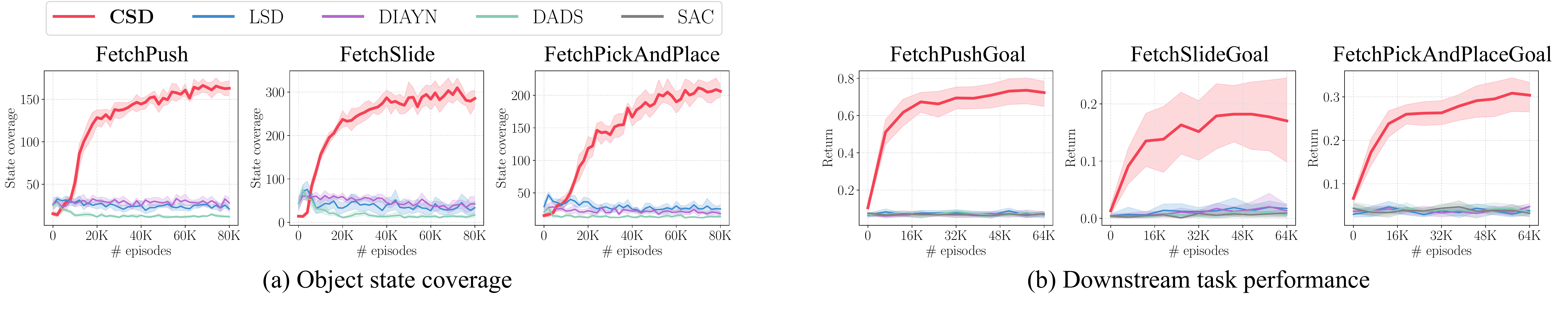}
    \vspace{-25pt}
    \caption{
    Comparison of the object state coverage and downstream task performances
    of skill discovery methods in three Fetch manipulation environments.
    Only CSD learns to manipulate the object without external supervision,
    while the other methods mainly focus on controlling the internal states (\Cref{fig:qual_fetch_all})
    because there is little incentive for them to discover more `challenging' skills.
    }
    \vspace{-10pt}
    \label{fig:fetch_coverage_down}
\end{figure*}

\cutsectionup
\section{Experiments}
\cutsectiondown
\label{sec:exp}

The goal of our experiments is to verify
whether our controllability-aware skill discovery method can learn complex, useful skills
without supervision in a variety of environments.
We test CSD on six environments across three different domains:
three Fetch manipulation environments (FetchPush, FetchSlide, and FetckPickAndPlace)~\citep{fetch_plappert2018},
Kitchen~\citep{kitchen_gupta2019},
and two MuJoCo locomotion environments (Ant and HalfCheetah)~\citep{mujoco_todorov2012,openaigym_brockman2016}.
We mainly compare CSD with three state-of-the-art unsupervised skill discovery methods:
LSD~\citep{lsd_park2022}, DIAYN~\citep{diayn_eysenbach2019}, and DADS~\citep{dads_sharma2020}.
They respectively fall into the categories of Euclidean distance-maximizing skill discovery,
reverse-MI, and forward-MI (\Cref{sec:related_work}).
We also compare with disagreement-based exploration used in unsupervised goal-conditioned RL,
such as LEXA~\citep{lexa_mendonca2021}, in \Cref{sec:appx_disag}.
We evaluate state coverage and performance on downstream tasks to assess the diversity and usefulness of the skills learned by each method.
For our quantitative experiments,
we use $8$ random seeds and present $95\%$ confidence intervals using error bars or shaded areas.
We refer to \csdlink for videos.

\cutsubsectionup
\subsection{Fetch Manipulation}
\cutsubsectiondown
\label{sec:exp_fetch}

We first show
(1)~whether CSD can acquire object manipulation skills without any supervision,
(2)~how useful the learned skills are for the downstream tasks,
and (3)~which component allows CSD to learn complex skills in the Fetch manipulation environments~\citep{fetch_plappert2018}.
Each Fetch environment consists of a robot arm and an object
but has a unique configuration;
\eg, FetchSlide has a slippery table and FetchPickAndPlace has a two-fingered gripper.

We train CSD, LSD, DIAYN, and DADS on the three Fetch environments for $80$K episodes
with 2-D continuous skills (FetchPush, FetchSlide) or 3-D continuous skills (FetchPickAndPlace).
Note that we do not leverage human prior knowledge on the state space (\eg, object pose);
thus, all methods are trained on the \emph{full state} in this experiment.\footnote{
We note that the Fetch experiments in the LSD paper~\citep{lsd_park2022} are using the `oracle' prior, which enforces an agent to only focus on the state change of the object.
}

\begin{figure}[t!]
    \centering
    \includegraphics[width=\linewidth]{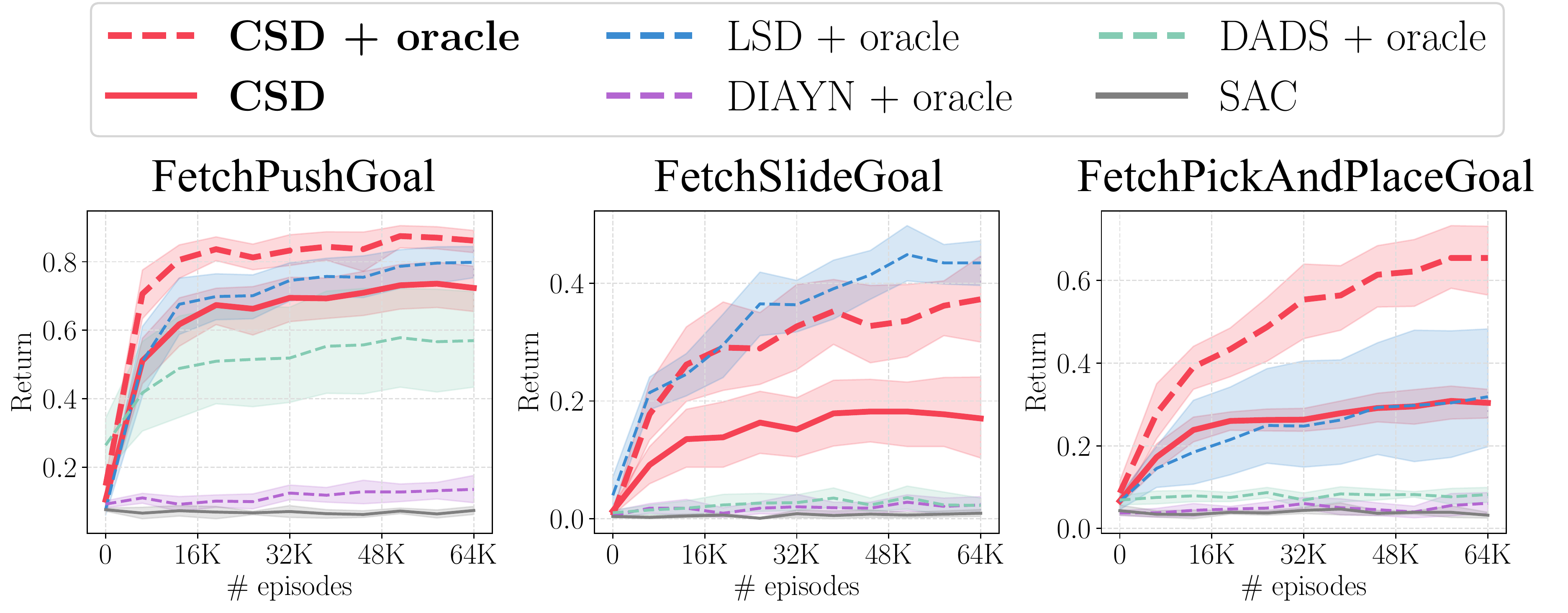}
    \vspace{-2em}
    \caption{
    Comparison of the downstream task performances of skill discovery methods with the oracle prior,
    which restricts the input to the skill discriminators to the object $xyz$ coordinates.
    }
    \vspace{-10pt}
    \label{fig:fetch_down_oracle}
\end{figure}

\Cref{fig:qual_fetch} illustrates the object trajectories of continuous skills learned
by skill discovery methods in the absence of any supervision.
CSD successfully learns to move the object in diverse directions without external supervision.
On the other hand, all of the previous methods
fail to learn such skills and instead focus on diversifying the joint angles of the robot arm itself.
This is because there is no incentive for the previous methods to focus on \emph{challenging} skills such as object manipulation,
while CSD explicitly finds hard-to-achieve state transitions.

Following the setup in \citet{lsd_park2022}, we evaluate two quantitative metrics:
the object state coverage and goal-reaching downstream task performance.
\Cref{fig:fetch_coverage_down}a compares the four skill discovery methods
in terms of the object state coverage,
which is measured by the number of $0.1 \times 0.1$ square bins occupied by the object at least once,
in the three Fetch environments.
\Cref{fig:fetch_coverage_down}b shows
the comparison of the goal-reaching downstream task performances,
where we train a hierarchical controller $\pi^h(z|s, g)$
that sequentially combines skills $z$ for the frozen skill policy $\pi(a|s, z)$
to move the object to a goal position $g$.
We additionally train the vanilla SAC baseline to verify the effectiveness of leveraging autonomously discovered skills.
We refer to \Cref{sec:appx_down} for further details.
On both quantitative metrics, CSD outperforms the prior methods by large margins,
successfully discovering diverse manipulation skills that are useful for solving downstream tasks.

\begin{figure}[t!]
    \centering
    \includegraphics[width=\linewidth]{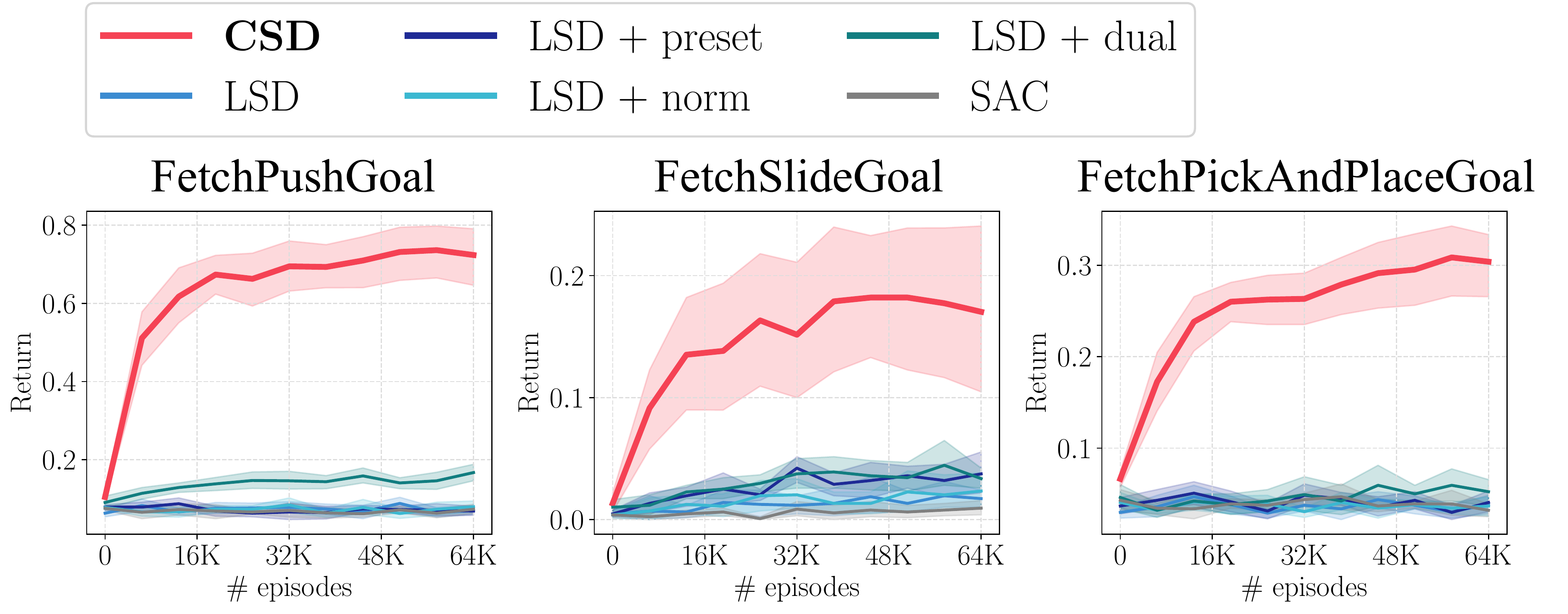}
    \vspace{-2em}
    \caption{
    Ablation study
    of distance-maximizing skill discovery in three Fetch environments.
    This suggests that CSD's performance cannot be achieved by
    just applying simple tricks to the previous Euclidean distance-maximizing skill discovery method.
    }
    \vspace{-15pt}
    \label{fig:fetch_down_ablation}
\end{figure}

\begin{figure*}[t!]
    \centering
    \includegraphics[width=\linewidth]{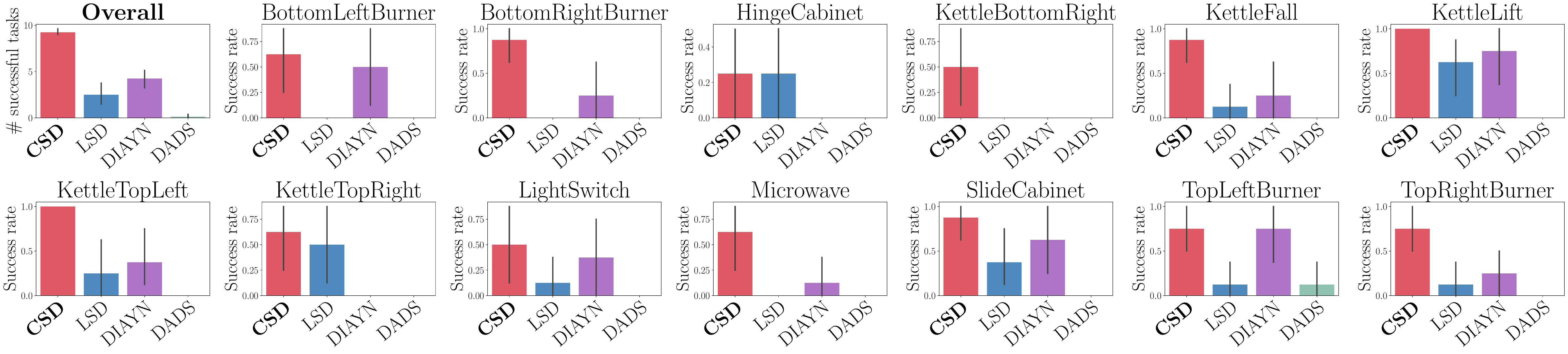}
    \vspace{-20pt}
    \caption{
    Task success rates of $16$ discrete skills discovered by CSD, LSD, DIAYN, and DADS in the Kitchen environment.
    CSD learns to manipulate diverse objects in the kitchen without any supervision.
    We refer to \Cref{sec:appx_add_results} for the results with 2-D continuous skills.
    }
    \vspace{-10pt}
    \label{fig:kitchen_coin_16}
\end{figure*}

\textbf{Skill discovery with the oracle prior on the state space.}
While our experiments show that our approach can discover useful manipulation skills without any human prior on the state space, %
previous unsupervised skill discovery methods~\citep{diayn_eysenbach2019,dads_sharma2020,lsd_park2022} mostly do not work without the \emph{oracle state prior},
which restricts the skill discriminator module's input to only the $xyz$ coordinates of the object.
To investigate how CSD and the prior methods perform in the presence of this supervision,
we train them with the oracle state prior.
\Cref{fig:fetch_down_oracle} demonstrates that even without the oracle state prior,
our CSD is mostly comparable to
the previous best method with the oracle prior.
This result demonstrates the potential of our approach in scalability to more complex environments, where human prior is no longer available.
Moreover, with the oracle state prior, CSD further improves its performance.
We refer to \Cref{fig:qual_fetch_oracle_all} for the full qualitative results of CSD and LSD with the oracle prior in FetchPickAndPlace.

\begin{figure}[t!]
    \centering
    \includegraphics[width=\linewidth]{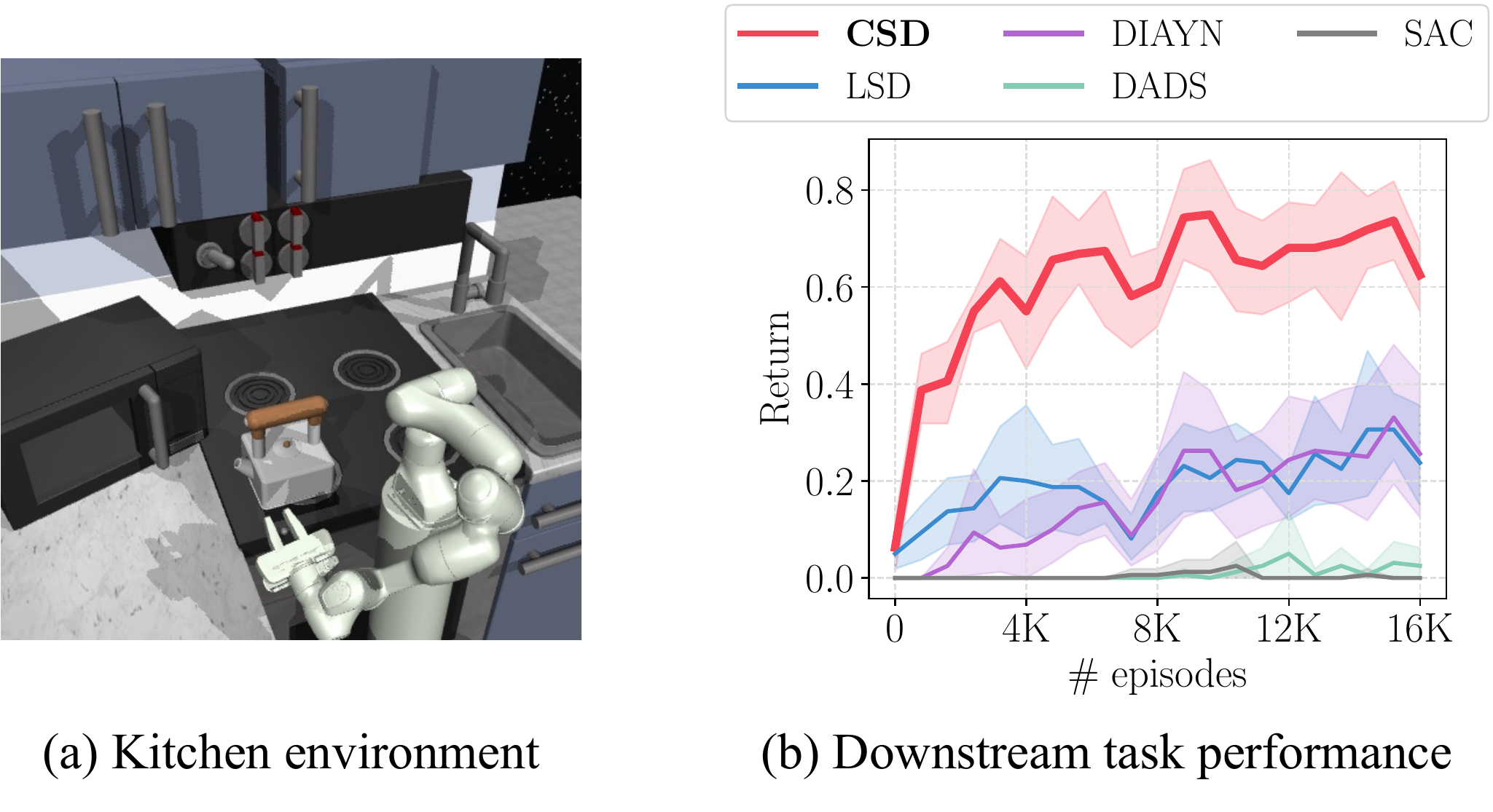}
    \vspace{-20pt}
    \caption{
    Comparison of the downstream task performances of skill discovery methods in the Kitchen environment.
    }
    \vspace{-10pt}
    \label{fig:kitchen_down}
\end{figure}

\begin{figure}[ht]
    \centering
    \includegraphics[width=\linewidth]{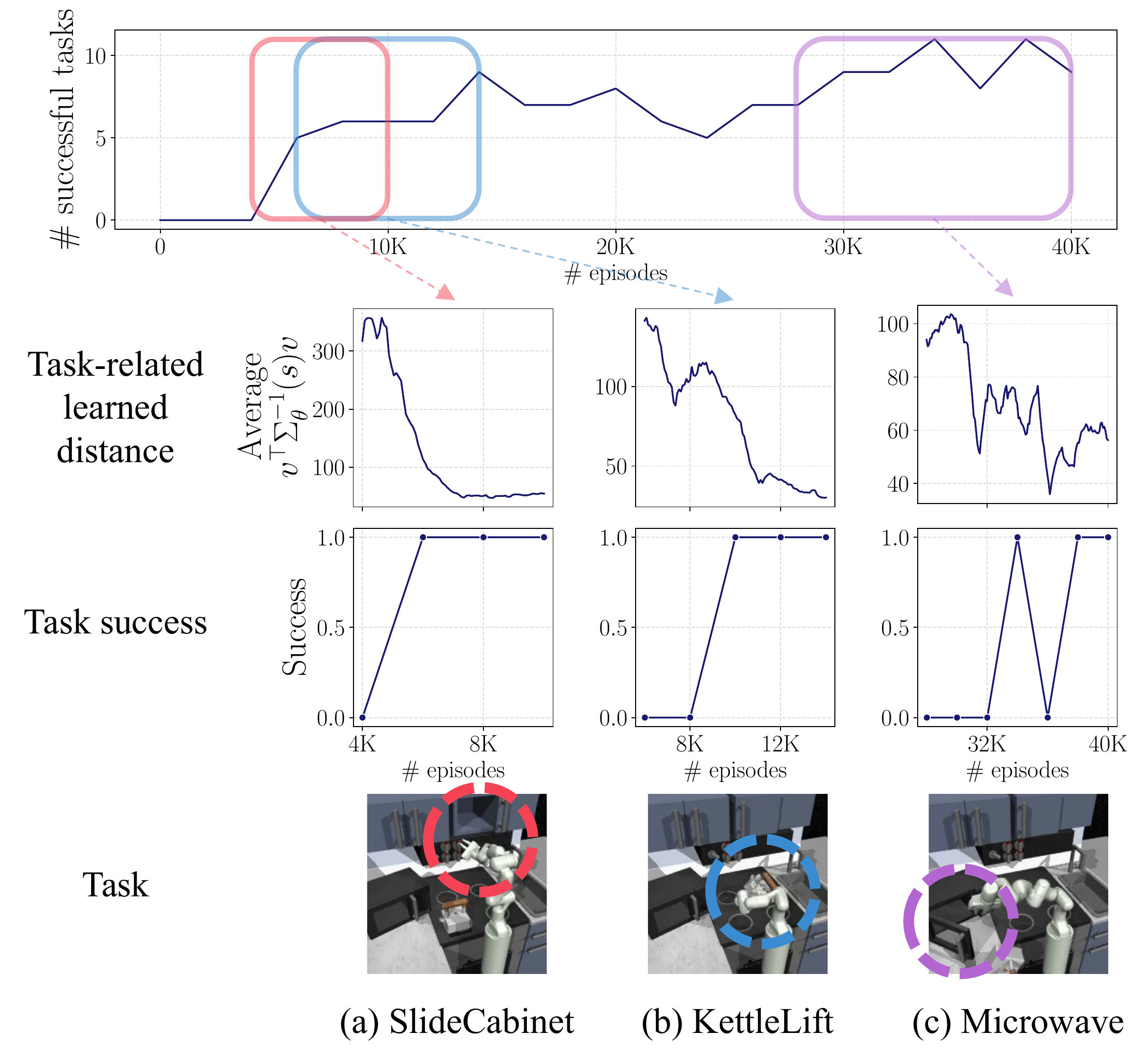}
    \vspace{-2em}
    \caption{
    Evolution of task-related distances and corresponding task success rates.
    Our learned task-related distances decrease once the agent gains control of the corresponding objects,
    which makes the agent focus on other new objects consistently over the course of training.
    Distance plots are smoothed over a window of size $10$ for better visualization.
    }
    \vspace{-10pt}
    \label{fig:kitchen_analysis}
\end{figure}

\textbf{Ablation study.}
To understand the importance of our controllability-aware distance function in CSD,
we examine whether similar results can be achieved without some components of CSD
or by just applying simple tricks to LSD, a previous Euclidean distance-maximizing skill discovery method.
Specifically, we consider the following three variants:
(1) LSD\,+\,preset:
LSD with a normalized state space
using the precomputed standard deviation of each state dimension from randomly generated trajectories,
(2) LSD\,+\,norm: LSD with a normalized state space using the moving average of the standard deviation of state differences ($s'-s$),
and (3) LSD\,+\,dual: LSD trained with dual gradient descent instead of spectral normalization
(\ie, CSD without our learned distance function).
\Cref{fig:fetch_down_ablation} compares the performances of these variants with CSD, LSD, and SAC in three downstream tasks.
The results show that only CSD learns to manipulate objects,
which suggests that our controllability-aware distance function is indeed necessary
to discover such complex skills without supervision.

\cutsubsectionup
\subsection{Kitchen Manipulation}
\cutsubsectiondown
\label{sec:exp_kitchen}

To verify the scalability of unsupervised skill discovery in a complex environment with diverse objects,
we evaluate our method on the Kitchen manipulation environment~\citep{kitchen_gupta2019},
which includes $13$ downstream tasks in total, such as opening a microwave, turning a light switch, moving a kettle, and opening slide/hinge cabinet doors (\Cref{fig:kitchen_down}a).
We train CSD, LSD, DIAYN, and DADS with both 2-D continuous skills and $16$ discrete skills
for $40$K episodes without any supervision.
We refer to \Cref{sec:appx_impl_details} for further experimental details regarding the Kitchen environment.

We first measure the task success rates of the skills learned by the four methods.
After the unsupervised skill training, we roll out the skill policy to collect $50$ trajectories with $50$ randomly sampled $z$s
and measure whether each of the $13$ tasks has at least one successful trajectory.
The results with $16$ discrete skills in \Cref{fig:kitchen_coin_16} suggest that CSD learns on average $10$ out of $13$ skills,
while the prior methods fail to discover such skills ($2$ for LSD, $4$ for DIAYN, $0$ for DADS)
because they mainly focus on diversifying the robot state.
Continuous skills in \Cref{fig:kitchen_coin_2} also show similar results.

We then evaluate the downstream task performance
by training a high-level controller $\pi^h(z | s, g)$ with the learned 2-D continuous skills $\pi(a | s, z)$
as behavioral primitives to achieve a task specified by a 13-D one-hot vector $g$.
The high-level controller chooses a skill $z$ every $10$ steps until the episode ends.
The results in \Cref{fig:kitchen_down}b show that CSD significantly outperforms the previous methods.

\textbf{Qualitative analysis.}
\Cref{fig:kitchen_analysis} illustrates how our controllability-aware distance evolves over time
and how this leads to the discovery of diverse, complex skills, \eg, SlideCabinet, KettleLift, and Microwave. %
Over training, we measure the task-related controllability-aware distance
$v^\top \Sigma_\theta^{-1}(s) v$ for each task $v$ using skill trajectories,
where $v$ is the one-hot task vector corresponding to each of the three tasks.
At around $4$K episodes (\Cref{fig:kitchen_analysis}a),
our controllability-aware distance encourages the agent to control the sliding cabinet with a large distance value (\ie, high reward).
Once the agent learns to manipulate the sliding cabinet door,
our controllability-aware distance for that skill decreases,
letting the agent move its focus to other harder-to-achieve skills,
\eg, lifting kettle (\Cref{fig:kitchen_analysis}b) or opening a microwave (\Cref{fig:kitchen_analysis}c).
As a result, the number of successful tasks gradually increases over the course of training.

\cutsubsectionup
\subsection{MuJoCo Locomotion}
\cutsubsectiondown
\label{sec:exp_mujoco}

\begin{figure}[t!]
    \centering
    \includegraphics[width=\linewidth]{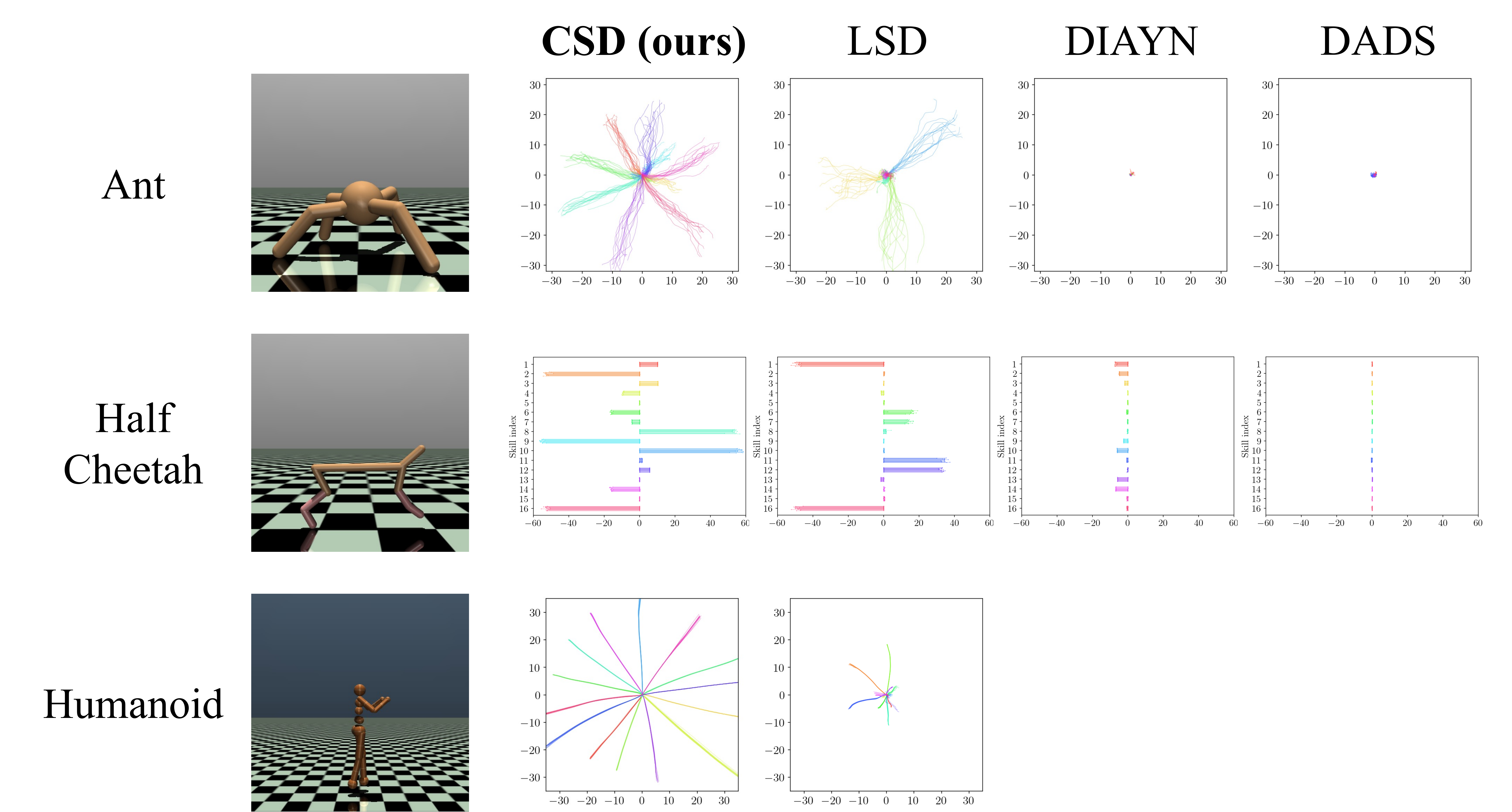}
    \vspace{-2em}
    \caption{
    The agent's $xy$ (Ant and Humanoid) or $x$ (HalfCheetah) trajectories of skills discovered by CSD, LSD, DIAYN, and DADS in MuJoCo locomotion environments.
    Trajectories with different colors represent different skills.
    We refer to \Cref{sec:appx_add_results} for the complete qualitative results
    from all random seeds.
    }
    \vspace{-10pt}
    \label{fig:qual_mujoco}
\end{figure}

To assess whether the idea of controllability-aware skill discovery works on domains
other than manipulation,
we evaluate CSD mainly on two MuJoCo locomotion environments \citep{mujoco_todorov2012,openaigym_brockman2016}: Ant and HalfCheetah.
We additionally employ $17$-DoF Humanoid, the most complex environment in the benchmark, for a qualitative comparison between CSD and LSD.
In these environments, we train skill discovery methods for $200$K episodes ($100$K for Humanoid) with $16$ discrete skills. %

\Cref{fig:qual_mujoco} shows examples of skills discovered by each method,
which suggests that CSD leads to the largest state coverage thanks to our controllability-aware distance function.
For quantitative evaluation, we first measure the state space coverage by counting the number of $1 \times 1$ bins occupied by the agent's $xy$ coordinates ($xz$ coordinates for 2-D HalfCheetah) at least once.
\Cref{fig:mj}a demonstrates that CSD covers the largest area among the four methods.
This is because CSD's controllability objective makes the agent mainly focus on diversifying the global position of the agent,
which corresponds to the `challenging' state transitions in these locomotion environments.
We emphasize that CSD not just learns to navigate in diverse directions but also learns a variety of behaviors,
such as rotating and flipping in both environments (\csdvideo).
We also note that MI-based methods (DIAYN and DADS) completely fail to diversify the agent's location
and only discover posing skills,
because the MI objective is agnostic to the distance metric, not providing incentives to maximize traveled distances in the state space.

\begin{figure}[t!]
    \centering
    \includegraphics[width=\linewidth]{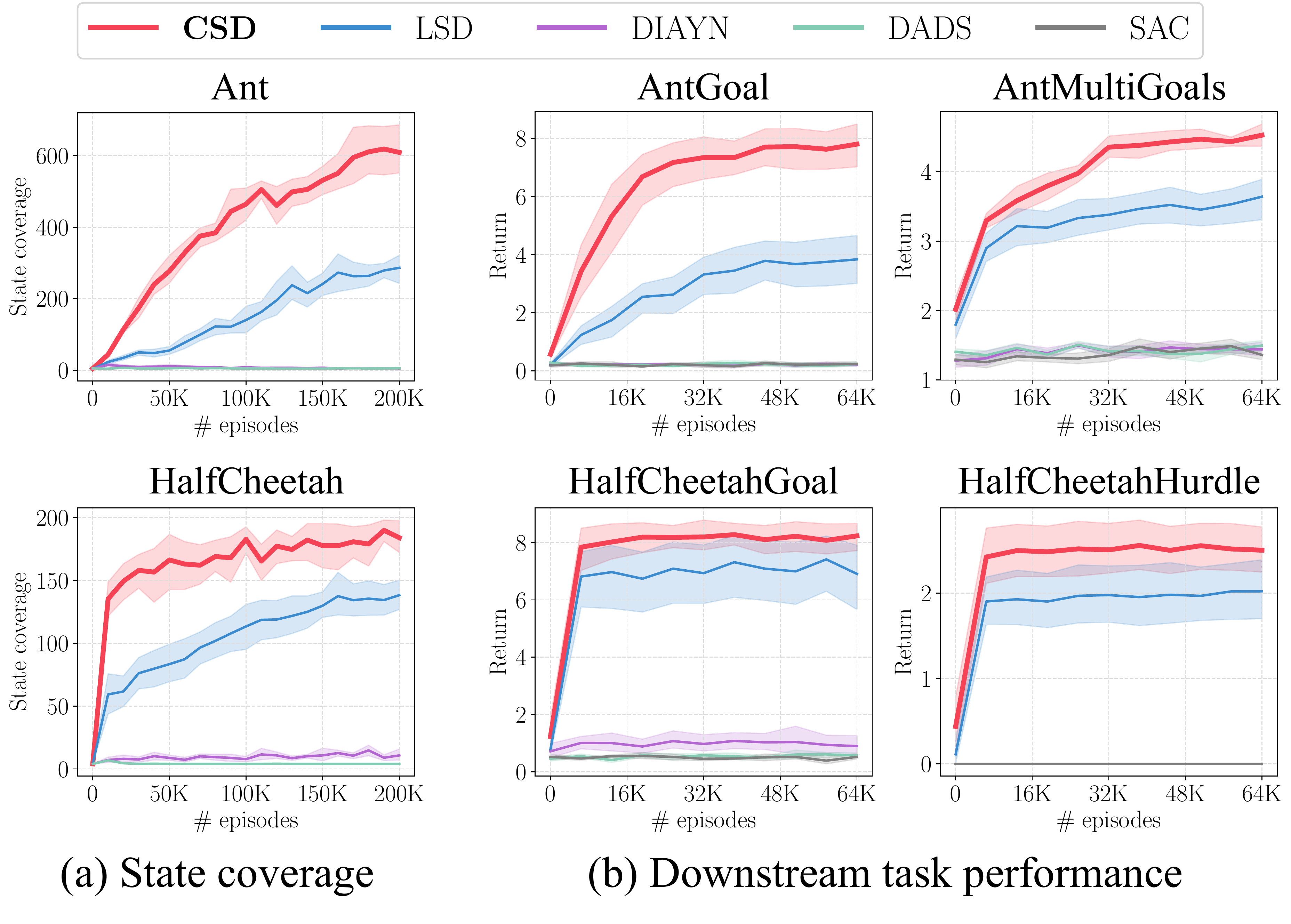}
    \vspace{-20pt}
    \caption{
    Comparison of the state coverage and downstream task performance of skills discovery methods in Ant and HalfCheetah.
    }
    \vspace{-10pt}
    \label{fig:mj}
\end{figure}

We also evaluate the downstream learning performance on four tasks:
AntGoal, AntMultiGoals, HalfCheetahGoal, and HalfCheetahHurdle,
following previous works \citep{diayn_eysenbach2019,dads_sharma2020,ibol_kim2021,lsd_park2022}.
In AntGoal and HalfCheetahGoal, the agent should reach a randomly sampled goal position,
and in AntMultiGoals, the agent should follow multiple randomly sampled goals in sequence.
In HalfCheetahHurdle \citep{hurdle_qureshi2020}, the agent should jump over as many hurdles as possible.
With downstream task rewards, we train a high-level policy that sequentially combines the learned skills.
In \Cref{fig:mj}b,
CSD consistently demonstrates the best performance among the four methods,
which suggests that the skills discovered by CSD are effective
not just on locomotion tasks but also on a wide variety of tasks, such as hurdle jumping.

\cutsectionup
\section{Conclusion}
\cutsectiondown

In this paper, we present Controllability-aware Skill Discovery (CSD), a novel unsupervised skill discovery method
that explicitly looks for hard-to-achieve skills.
Specifically, we first formulate a distance-maximizing skill discovery approach (DSD),
which can be combined with any arbitrary distance function.
We then propose a jointly trained controllability-aware distance function,
which consistently encourages the agent to discover more complex, hard-to-achieve skills.
We empirically show that the idea of controllability-awareness enables the agent to acquire diverse complex skills
in the absence of supervision in a variety of robotic manipulation and locomotion environments.

\textbf{Limitations and future directions.}
Although the general idea of controllability-aware skill discovery is still applicable to pixel domains,
\eg, in combination with representation learning techniques \citep{dreamer_hafner2020,curl_srinivas2020,mwm_seo2022},
where they will reveal both the object and agent representations and CSD will focus on the object representation,
we did not verify the scalability of our controllability-aware distance function
to pixel-based environments.
We leave it as future work.
Another limitation is that CSD in its current form might not discover `slowly moving' skills
because underlying DSD prefers skills with large state variations.
We believe acquiring skills with diverse moving speeds is another interesting future direction.

\section*{Acknowledgement}
We would like to thank Amber Xie, Younggyo Seo, and Jaekyeom Kim for their insightful feedback and discussion.
This work was funded in part by Darpa RACER, Komatsu, a Berkeley Graduate Fellowship, and the BAIR Industrial Consortium.
Seohong Park was partly supported by Korea Foundation for Advanced Studies (KFAS).

\bibliography{csd}
\bibliographystyle{icml2023}

\newpage
\appendix
\onecolumn
\section{Extended Related Work on Unsupervised RL}
\label{sec:appx_related}

The goal of unsupervised RL is to learn useful knowledge,
such as dynamics models, state representations, and behavioral primitives,
without predefined tasks
so that we can later utilize them to efficiently solve downstream tasks.
One line of research focuses on gathering knowledge of the environment with pure exploration
\citep{icm_pathak2017,rnd_burda2019,disag_pathak2019,p2e_sekar2020,apt_liu2021,protorl_yarats2021,umbrl_rajeswar2022}.
Unsupervised skill discovery methods
\citep{vic_gregor2016,sectar_coreyes2018,diayn_eysenbach2019,dads_sharma2020,ibol_kim2021,upside_kamienny2022,disdain_strouse2022,lsd_park2022,disk_shafiullah2022,rest_jiang2022,mos_zhao2022}
aim to learn a set of temporally extended useful behaviors,
and our CSD falls into this category.
Another line of work focuses on discovering \emph{goals} and corresponding goal-conditioned policies
via pure exploration \citep{discern_wardefarley2019,skewfit_pong2020,mega_pitis2020,lexa_mendonca2021}
or asymmetric/curriculum self-play \citep{asp_sukhbaatar2018,asp_openai2021,cusp_du2022}.
Lastly, \citet{fb_touati2021,zs_touati2022} aim to learn a set of policies that can be instantly adapted to task reward functions
given an unsupervised exploration method or an offline dataset.

\section{Theoretical Results}

\subsection{Proof of \Cref{thm:dsd}}
\label{sec:appx_proof}

We assume that we are given an arbitrary non-negative function $d: \gS \times \gS \to \sR^+_0$.
We first introduce some additional notations.
For $x, y \in \gS$, define $d_s(x, y) \triangleq \min(d(x, y), d(y, x))$.
For $x, y \in \gS$, let $P(x, y)$ be the set of all finite state paths from $x$ to $y$.
For a state path $p = (s_0, s_1, \dots, s_t)$,  %
define $D_s(p) \triangleq \sum_{i=0}^{t-1} d_s(s_i, s_{i+1})$.

Now, for $x, y \in \gS$, we define the \emph{induced pseudometric} $\tilde{d}: \gS \times \gS \to \sR^+_0$ as follows:
\begin{align}
    \tilde{d}(x, y) \triangleq
    \begin{cases}
    \inf_{p \in P(x, y)} D_s(p) & \text{if } x \neq y \\
    0               & \text{if } x = y
    \end{cases}.
\end{align}

Then, the following theorems hold.

\begin{lemma} \label{thm:lowerbound}
$\tilde{d}$ is a lower bound of $d$, \ie,
\begin{align}
\forall x, y \in \gS, \quad 0 \leq \tilde{d}(x, y) \leq d(x, y).
\end{align}
\end{lemma}
\begin{proof}
If $x = y$, then $\tilde{d}(x, y) = 0$ by definition and thus $0 \leq \tilde{d}(x, y) \leq d(x, y)$ always holds.
Otherwise, $0 \leq \tilde{d}(x, y) \leq D_s((x, y)) = d_s(x, y) \leq d(x, y)$ holds and this completes the proof.
\end{proof}

\begin{theorem}
For $\phi: \gS \to \sR^D$,
imposing \Cref{eq:dsd_cst} with $d$ is equivalent to imposing \Cref{eq:dsd_cst} with $\tilde{d}$, \ie,
\begin{align}
\forall x, y \in \gS, \quad \|\phi(x)-\phi(y)\| \leq d(x, y) \iff \forall x, y \in \gS, \quad \|\phi(x)-\phi(y)\| \leq \tilde{d}(x, y).
\end{align}
\end{theorem}
\begin{proof}
From \Cref{thm:lowerbound}, we know that $\|\phi(x)-\phi(y)\| \leq \tilde{d}(x, y)$ implies $\|\phi(x)-\phi(y)\| \leq d(x, y)$.
Now, we assume that $\|\phi(x)-\phi(y)\| \leq d(x, y)$ holds for any $x, y \in \gS$.
First, if $x = y$, then $\|\phi(x) - \phi(y)\|$ becomes $0$ and thus $\|\phi(x)-\phi(y)\| \leq \tilde{d}(x, y)$ always holds.
For $x \neq y$,
let us consider any state path $p = (s_0=x, s_1, s_2, \dots, s_{t-1}, s_t=y) \in P(x, y)$.
For any $i \in \{0, 1, \ldots, t-1\}$, we have
\begin{align}
    \|\phi(s_i) - \phi(s_{i+1})\| &\leq d(s_i, s_{i+1}), \\
    \|\phi(s_{i+1}) - \phi(s_i)\| &\leq d(s_{i+1}, s_i),
\end{align}
and thus we get $\|\phi(s_i) - \phi(s_{i+1})\| = \|\phi(s_{i+1}) - \phi(s_i)\| \leq \min(d(s_i, s_{i+1}),  d(s_{i+1}, s_i)) = d_s(s_i, s_{i+1})$. Now, we have the following inequalities:
\begin{align}
    \|\phi(s_0) - \phi(s_1)\| &\leq d_s(s_0, s_1), \\
    \|\phi(s_1) - \phi(s_2)\| &\leq d_s(s_1, s_2), \\
    \dots, \\
    \|\phi(s_{t-1}) - \phi(s_t)\| &\leq d_s(s_{t-1}, s_t).
\end{align}
From these, we obtain $\|\phi(x) - \phi(y)\| = \|\phi(s_0) - \phi(s_t)\| \leq \sum_{i=0}^{t-1} \|\phi(s_i) - \phi(s_{i+1})\|
\leq \sum_{i=0}^{t-1} d_s(s_i, s_{i+1}) = D_s(p)$.
Then, by taking the infimum of the right-hand side over all possible $p \in P(x, y)$,
we get $\|\phi(x) - \phi(y)\| \leq \inf_{p \in P(x, y)} D_s(p) = \tilde{d}(x, y)$ and this completes the proof.
\end{proof}

\begin{theorem}
$\tilde{d}$ is a valid pseudometric, \ie,
\begin{enumerate}[label=(\alph*)]
    \item $\forall x \in \gS$, $\tilde{d}(x, x) = 0$.
    \item (Symmetry) $\forall x, y \in \gS$, $\tilde{d}(x, y) = \tilde{d}(y, x)$.
    \item (Triangle inequality) $\forall x, y, z \in \gS$, $\tilde{d}(x, y) \leq \tilde{d}(x, z) + \tilde{d}(z, y)$.
\end{enumerate}
\end{theorem}
\begin{proof}
(\emph{a}) By definition, $\tilde{d}(x, x) = 0$ always holds for all $x \in \gS$.

(\emph{b}) If $x = y$, then $\tilde{d}(x, y) = \tilde{d}(y, x) = 0$.
Otherwise, with $p = (s_0=x, s_1, s_2, \dots, s_{t-1}, s_t=y) \in P(x, y)$, we can prove the symmetry of $\tilde{d}$ as follows:
\begin{align}
    \tilde{d}(x, y) &= \inf_{p \in P(x, y)} D_s(p) \\
    &= \inf_{p \in P(x, y)} \sum_{i=0}^{t-1} d_s(s_i, s_{i+1}) \\
    &= \inf_{p \in P(x, y)} \sum_{i=0}^{t-1} d_s(s_{i+1}, s_i) \\
    &= \inf_{p \in P(y, x)} D_s(p) \\
    &= \tilde{d}(y, x).
\end{align}

(\emph{c}) If $x = y$, $y = z$, or $z = x$, then it can be easily seen that
$\tilde{d}(x, y) \leq \tilde{d}(x, z) + \tilde{d}(z, y)$ always holds.
Hence, we assume that they are mutually different from each other.
Then, the following inequality holds:
\begin{align}
    \tilde{d}(x, y) &= \inf_{p \in P(x, y)} D_s(p) \\
    &\leq \inf_{p_1 \in P(x, z), p_2 \in P(z, y)} D_s(p_1) + D_s(p_2) \\
    &= \inf_{p_1 \in P(x, z)} D_s(p_1) + \inf_{p_2 \in P(z, y)} D_s(p_2) \\
    &= \tilde{d}(x, z) + \tilde{d}(z, y),
\end{align}
which completes the proof.
\end{proof}

\subsection{Implications of \Cref{thm:dsd}}
\label{sec:appx_implication}

\Cref{thm:dsd} suggests that
the constraint in \Cref{eq:dsd_cst} implicitly transforms an arbitrary distance function $d$ into a tighter valid pseudometric $\tilde{d}$.
Intuitively, this $\tilde{d}(x, y)$ corresponds to the minimum possible (symmetrized) path distance from $x$ to $y$.
Hence, if we train DSD with \Cref{eq:dsd_cst},
it will find long-distance transitions that cannot be equivalently achieved
by taking multiple short-distance transitions.
Intuitively, in the context of CSD (\Cref{sec:learned_distance}), 
this implies that the agent will find rare state transitions that cannot be bypassed by taking `easy' intermediate steps,
which is a desirable property.

However, there are some limitations regarding the use of our distance function $d^{\text{CSD}}$ (\Cref{eq:csd_dist}).
First, while the DSD constraint in \Cref{eq:dsd_cst} implicitly symmetrizes
the distance function by taking the minimum between $d(x, y)$ and $d(y, x)$,
this may not be ideal in highly asymmetric environments involving many irreversible transitions.
In practice, this may be resolved by only imposing one-sided constraints of our interest.
Second, in our implementation, we only consider a single-step transition $(s, s')$ and a single-step density model $q_\theta(s'|s)$
as we found this simple design choice to be sufficient for our experiments.
However, in order to fully leverage the aforementioned property of the induced pseudometric,
the constraint may be imposed on any state pairs with a multi-step density model,
which we leave for future work.

\section{Comparison with Unsupervised Disagreement-Based Exploration}
\label{sec:appx_disag}
\begin{figure*}[t!]
    \centering
    \includegraphics[width=\linewidth]{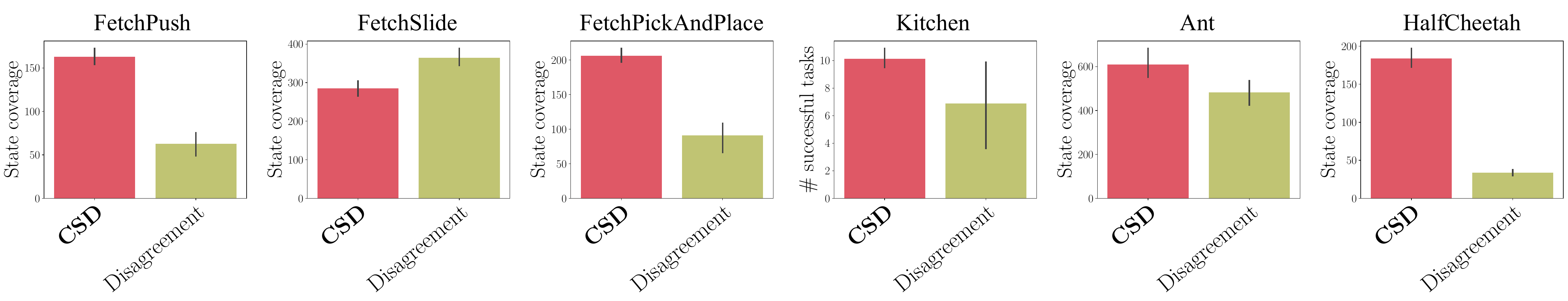}
    \caption{
    Comparison of unsupervised state coverage metrics between CSD and ensemble disagreement-based exploration \citep{disag_pathak2019}
    in all six environments.
    CSD mostly outperforms disagreement-based exploration in our state coverage metrics
    mainly because it actively diversifies hard-to-control states such as the object position or the agent location.
    }
    \label{fig:csd_disag}
\end{figure*}
\begin{figure*}[t!]
    \centering
    \includegraphics[width=0.9\linewidth]{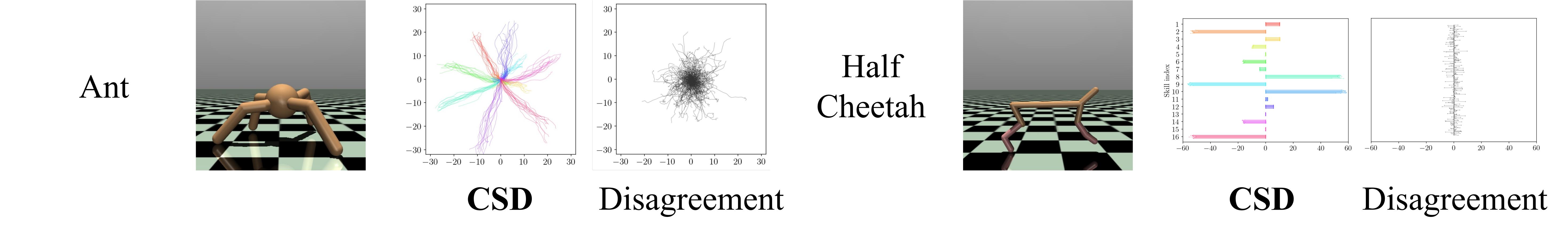}
    \caption{
    The agent's $xy$ (Ant) or $x$ (HalfCheetah) trajectories of CSD and disagreement-based exploration.
    While CSD seeks very consistent, directed behaviors, %
    disagreement-based exploration only focuses on diversifying states with chaotic, random behaviors.
    We provide videos illustrating this difference on \csdlink.
    }
    \label{fig:qual_csd_disag}
\end{figure*}

In this section,
we discuss the difference between CSD and unsupervised goal-conditioned RL
and present an empirical comparison between them.
Unsupervised goal-conditioned RL approaches,
such as DISCERN \citep{discern_wardefarley2019}, Skew-Fit \citep{skewfit_pong2020},
MEGA \citep{mega_pitis2020}, and LEXA \citep{lexa_mendonca2021},
learn diverse behaviors typically by
(1) running an exploration method that collects diverse `goal' states $g$
and (2) learning a goal-conditioned policy $\pi(a|s, g)$ to reach the states discovered.
Hence, treating $g$ as a $|\gS|$-dimensional skill latent vector,
these approaches may be viewed as a special type of unsupervised skill discovery.

However, the main focuses of unsupervised skill discovery are different from
that of unsupervised goal-conditioned RL.
First, unsupervised skill discovery aims to discover more general skills not restricted to goal-reaching behaviors,
which tend to be \emph{static} as the agent is encouraged to stay still at the goal state
\citep{lexa_mendonca2021,rest_jiang2022}.
For instance, our approach maximizes traveled distances,
which leads to more `dynamic' behaviors like consistently running in a specific direction (\Cref{fig:qual_mujoco}).
Second, unsupervised skill discovery aims to build a \emph{compact} set of skills, which could also be discrete,
rather than finding all the possible states in the given environment.
For example, if we train CSD with three discrete skills,
these behaviors will be as `distant' as possible from one another,
being maximally distinguishable.
As such, we can have useful behaviors with a much low-dimensional skill space,
making it more amenable to hierarchical RL.

Despite the difference in goals,
to better illustrate the difference between them,
we make an empirical comparison between CSD and ensemble disagreement-based exploration \citep{disag_pathak2019},
which some previous unsupervised goal-conditioned RL methods like LEXA \citep{lexa_mendonca2021} use as the exploration method.
Disagreement-based exploration learns an ensemble of $E$ forward dynamics models $\{\hat{p}_i(s'|s, a)\}_{i \in \{1, 2, \dots, E\}}$,
and uses its variance $\sum_{k}^{|\gS|} \sV[\hat{p}_i(\cdot_k|s, a)]$ as an intrinsic reward,
in order to seek unexplored transitions with high epistemic uncertainty.
While unsupervised goal-condition RL approaches additionally learn a goal-conditioned policy, we do not separately learn it
since the state coverage metrics of the exploration policy
can serve as an approximate upper bound of the corresponding optimal goal-conditioned policy's performance.

\Cref{fig:csd_disag} presents the comparisons of unsupervised state coverage metrics between CSD and disagreement-based exploration
in all of our six environments.
The results suggest that CSD mostly outperforms disagreement-based exploration in our state coverage metrics,
mainly because CSD actively diversifies hard-to-control states such as the object position or the agent location,
while the pure exploration method only focuses on finding unseen transitions.
This difference is especially prominent in Ant and HalfCheetah (\Cref{fig:qual_csd_disag}),
in which CSD seeks very consistent, directed behaviors, such as moving in one direction,
while disagreement-based exploration only focuses on diversifying states with chaotic, random behaviors.
We provide videos illustrating this difference at \csdaddress.

\section{Additional Results}
\label{sec:appx_add_results}

\begin{figure*}[t!]
    \centering
    \includegraphics[width=\linewidth]{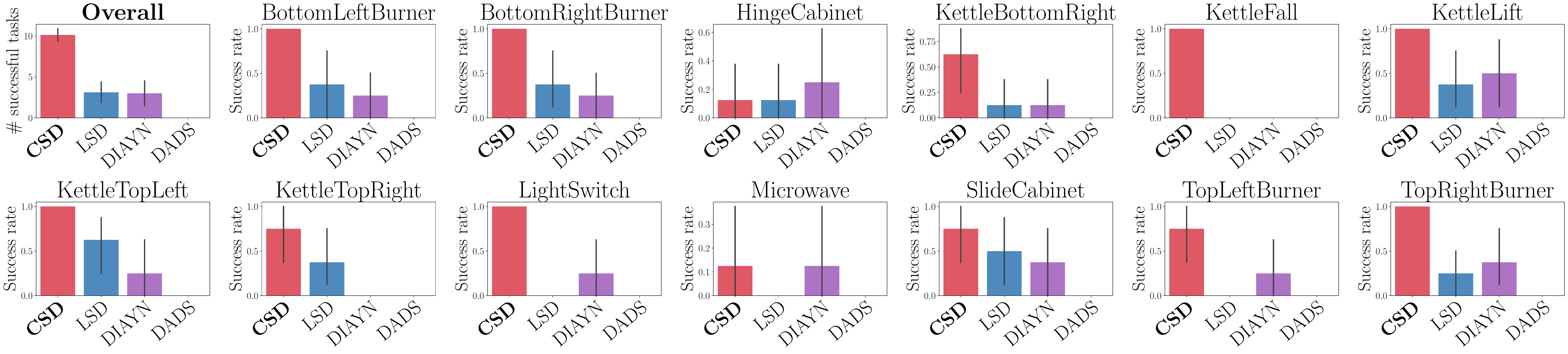}
    \caption{
    Task success rates of 2-D continuous skills discovered by four methods in the Kitchen environment.
    }
    \label{fig:kitchen_coin_2}
\end{figure*}
\begin{figure*}[t!]
    \centering
    \includegraphics[width=0.7\linewidth]{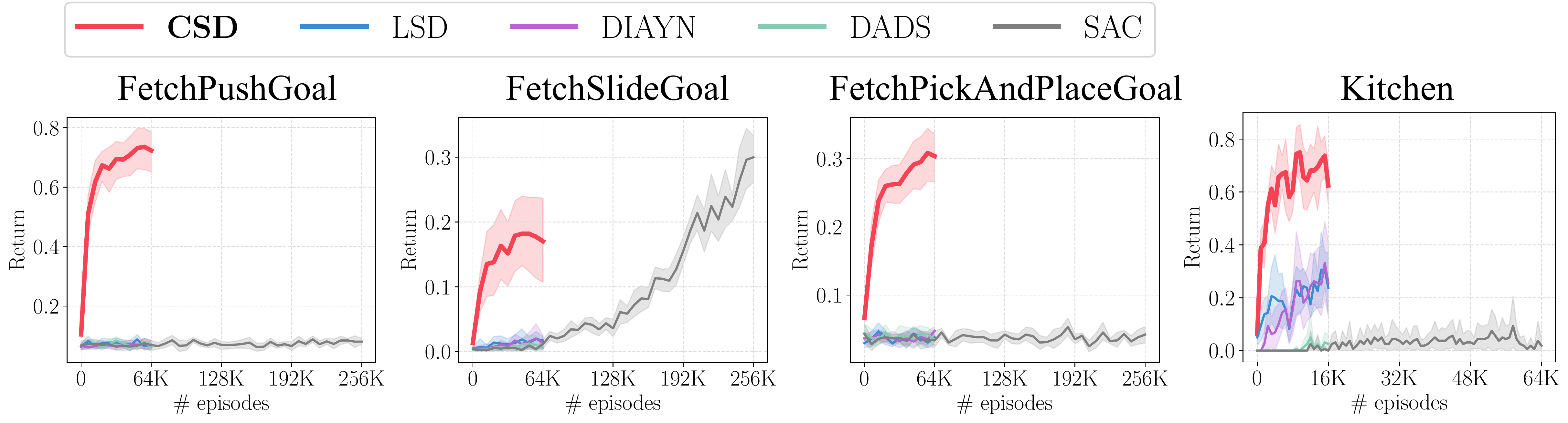}
    \caption{
    Extended learning curves of the SAC baseline in Fetch and Kitchen downstream tasks.
    }
    \label{fig:down_long}
\end{figure*}

\paragraph{Additional quantitative results.}
\Cref{fig:kitchen_coin_2} shows the task success rates of the 2-D continuous skills
learned by CSD, LSD, DIAYN, and DADS.
As in the discrete case, CSD outperforms the other methods by a significant margin.
\Cref{fig:down_long} demonstrates extended learning curves in Fetch and Kitchen downstream tasks,
where we train SAC for four times as long as skill discovery methods.
The results suggest that, while SAC alone can solve the FetchSlideGoal task with a lot more samples,
it fails at learning FetchPushGoal, FetchPickAndPlaceGoal, and Kitchen mainly because they are challenging sparse-reward tasks.
In contrast, agents can quickly learn all these tasks with temporally extended skills from CSD.

\paragraph{Additional qualitative results.}
\Cref{fig:qual_fetch_all,fig:qual_mujoco_all} illustrate
the skill trajectories of all runs we use for our experiments in Fetch manipulation and two MuJoCo locomotion environments
(eight random seeds for each method in each environment).
In the Fetch environments, CSD is the only method that learns object manipulation skills without supervision (\Cref{fig:qual_fetch_all}).
In Ant and HalfCheetah, CSD not only learns locomotion skills but also discovers a variety of diverse skills,
such as rotating and flipping in both environments (\Cref{fig:qual_mujoco_all}, \csdvideo).
We provide the complete qualitative results in Humanoid in \Cref{fig:qual_hum_all}.
\Cref{fig:qual_fetch_oracle_all} shows the full results of CSD and LSD equipped with the oracle prior
in FetchPickAndPlace (eight seeds each).
While CSD always learns to pick up the object, LSD discovers such skills in only three out of eight runs
(\Cref{fig:qual_fetch_oracle_all}).
This is because our controllability-aware distance function
consistently encourages the agent to learn more challenging picking-up behaviors.
As a result, CSD significantly outperforms LSD in downstream tasks (\Cref{fig:fetch_down_oracle}).

\begin{figure*}[t!]
    \centering
    \begin{subfigure}[ht]{1.0\textwidth}
        \centering
        \includegraphics[width=\linewidth]{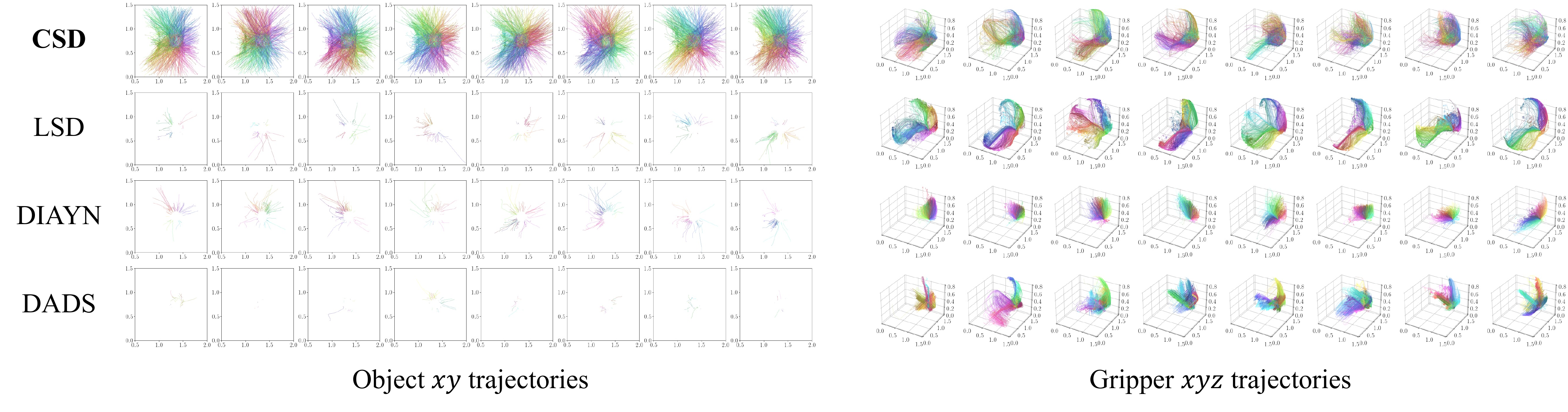}
        \caption{FetchPush}
        \vspace{15pt}
    \end{subfigure}
    \begin{subfigure}[ht]{1.0\textwidth}
        \centering
        \includegraphics[width=\linewidth]{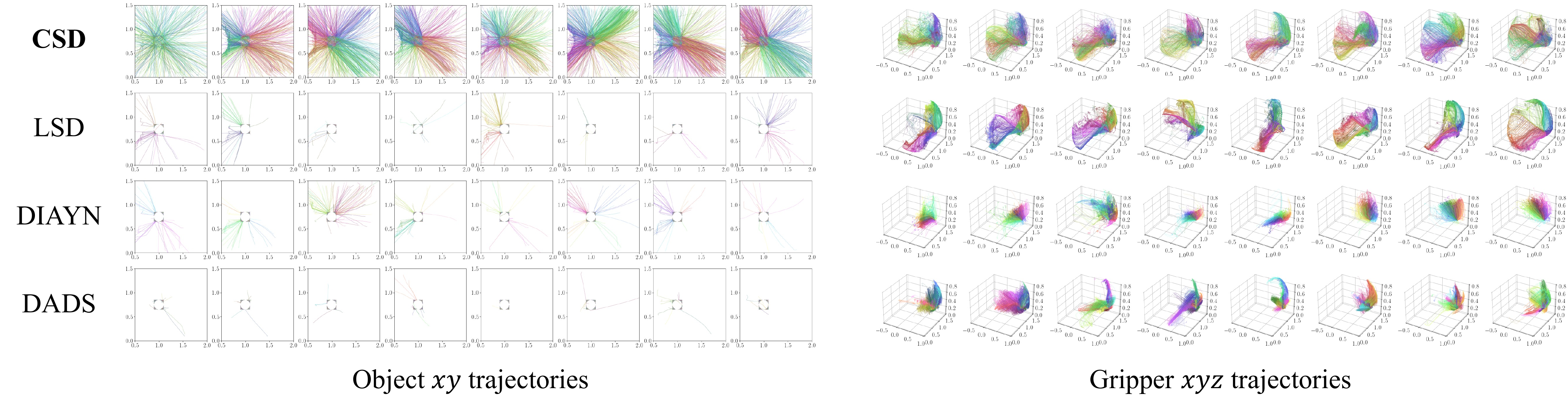}
        \caption{FetchSlide}
        \vspace{15pt}
    \end{subfigure}
    \begin{subfigure}[ht]{1.0\textwidth}
        \centering
        \includegraphics[width=\linewidth]{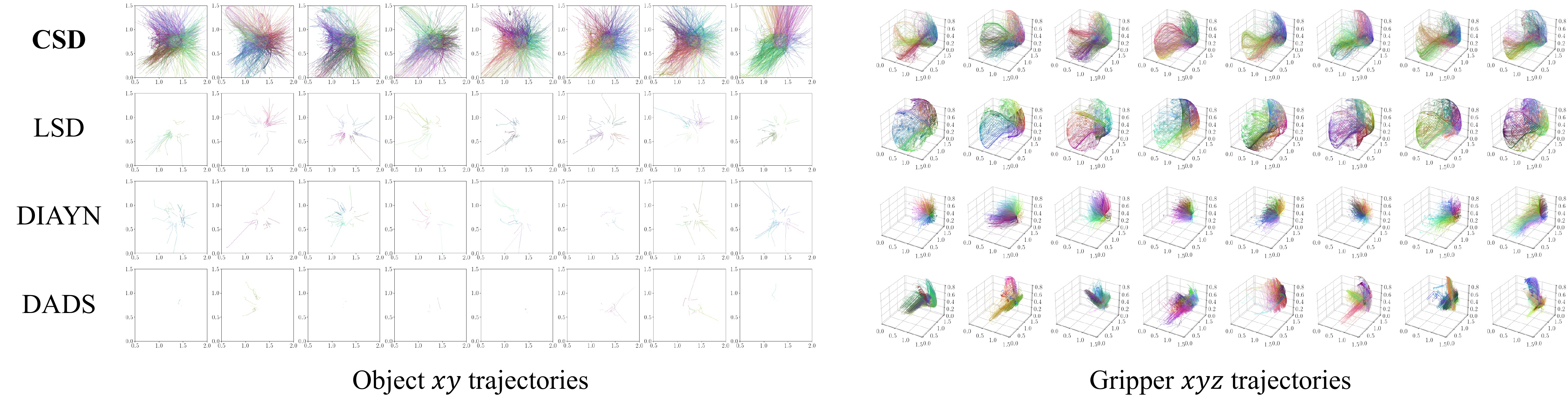}
        \caption{FetchPickAndPlace}
    \end{subfigure}
    \caption{
    Complete qualitative results in three Fetch environments (eight runs for each method in each environment).
    We plot the skill trajectories of the object and the gripper with different colors.
    CSD is the only unsupervised skill discovery method that discovers object manipulation skills without supervision.
    }
    \label{fig:qual_fetch_all}
\end{figure*}
\begin{figure*}[t!]
    \centering
    \includegraphics[width=\linewidth]{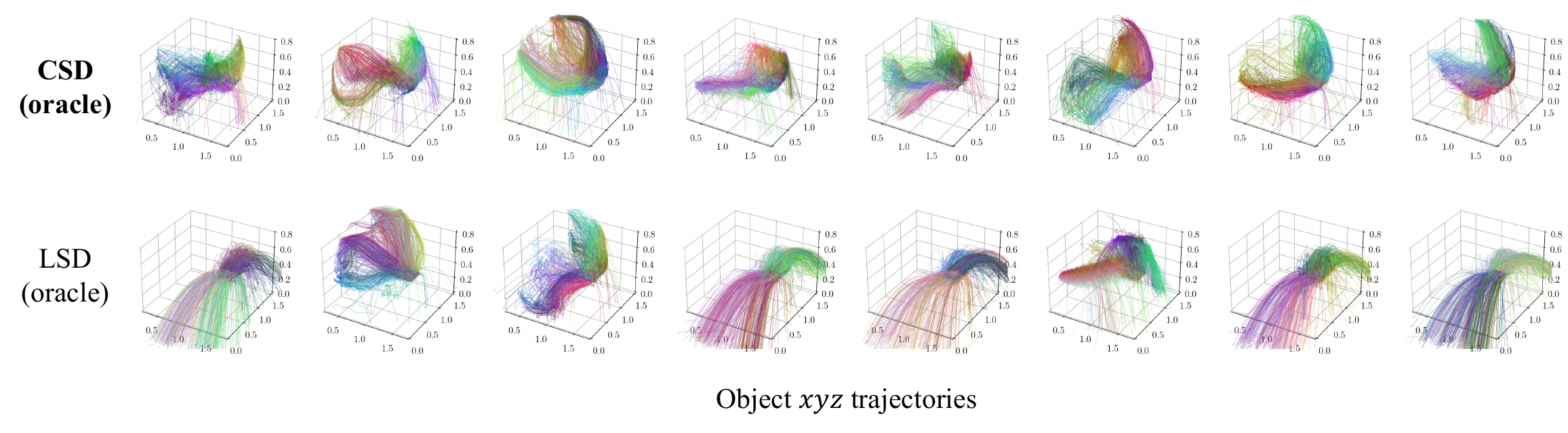}
    \caption{
    Complete qualitative results of CSD and LSD trained with the oracle prior in FetchPickAndPlace (eight runs for each method).
    We plot the skill trajectories of the object and the gripper with different colors.
    Note that while LSD mostly just throws the object away, CSD always learns to pick up the object in all eight runs.
    }
    \label{fig:qual_fetch_oracle_all}
\end{figure*}
\begin{figure*}[t!]
    \centering
    \includegraphics[width=0.6\linewidth]{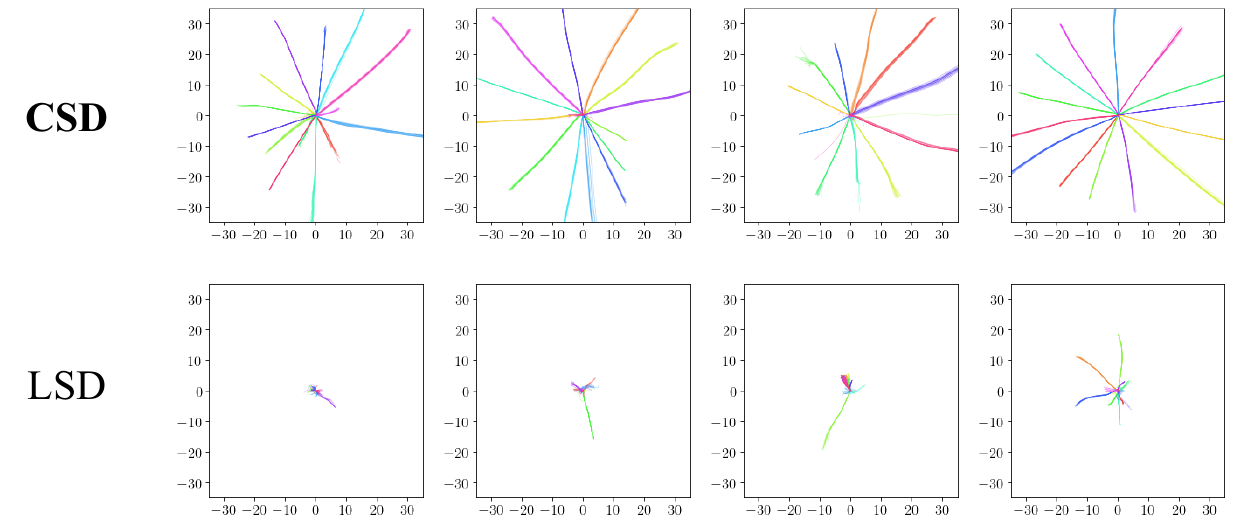}
    \caption{
    Complete qualitative results in Humanoid (four runs for each method in each environment).
    We plot the skill $xy$ trajectories of the agent with different colors.
    We note that we train CSD and LSD for $100$K episodes (which is a tenth of the number of episodes used in the LSD paper \citep{lsd_park2022}).
    }
    \label{fig:qual_hum_all}
\end{figure*}
\begin{figure*}[t!]
    \centering
    \begin{subfigure}[ht]{1.0\textwidth}
        \centering
        \includegraphics[width=\linewidth]{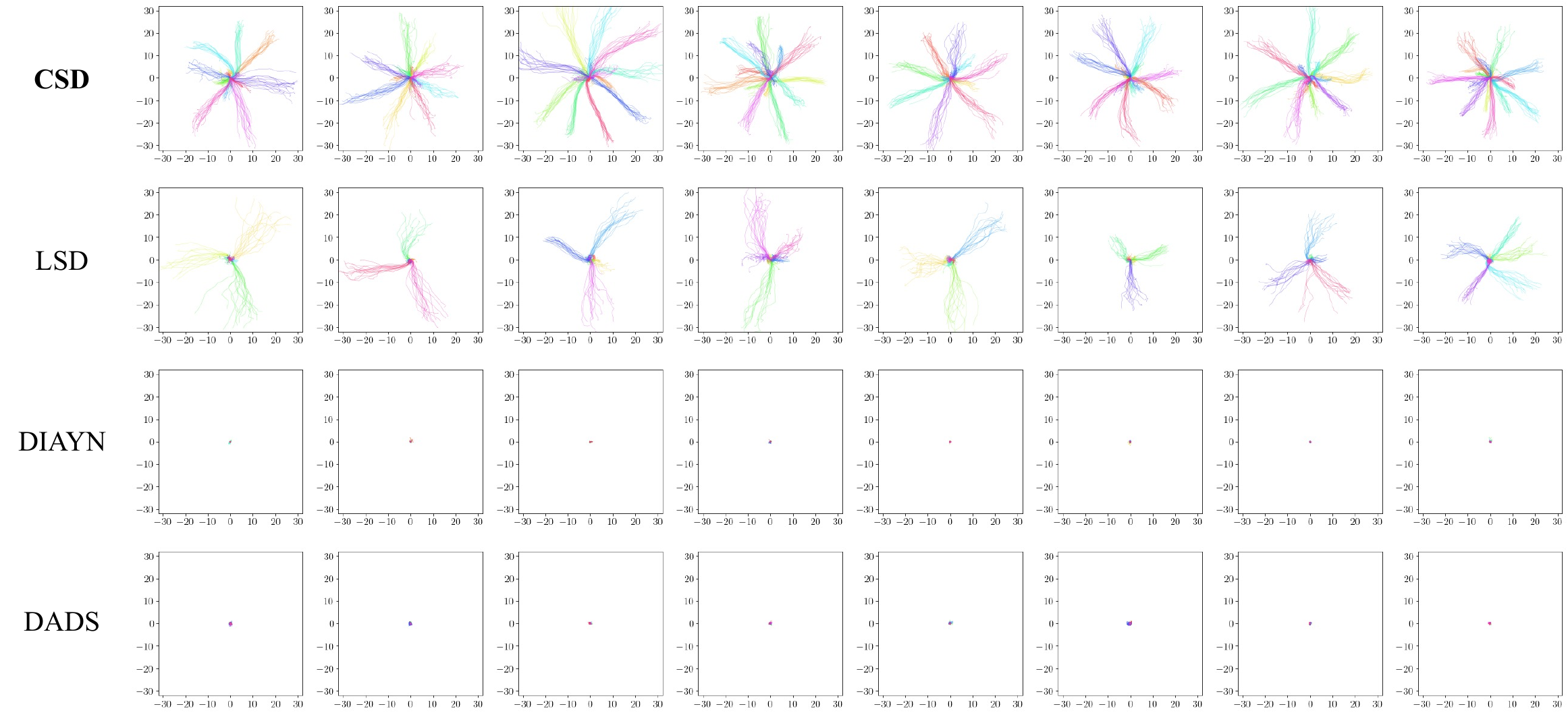}
        \caption{Ant $xy$ trajectories}
        \vspace{15pt}
    \end{subfigure}
    \begin{subfigure}[ht]{1.0\textwidth}
        \centering
        \includegraphics[width=\linewidth]{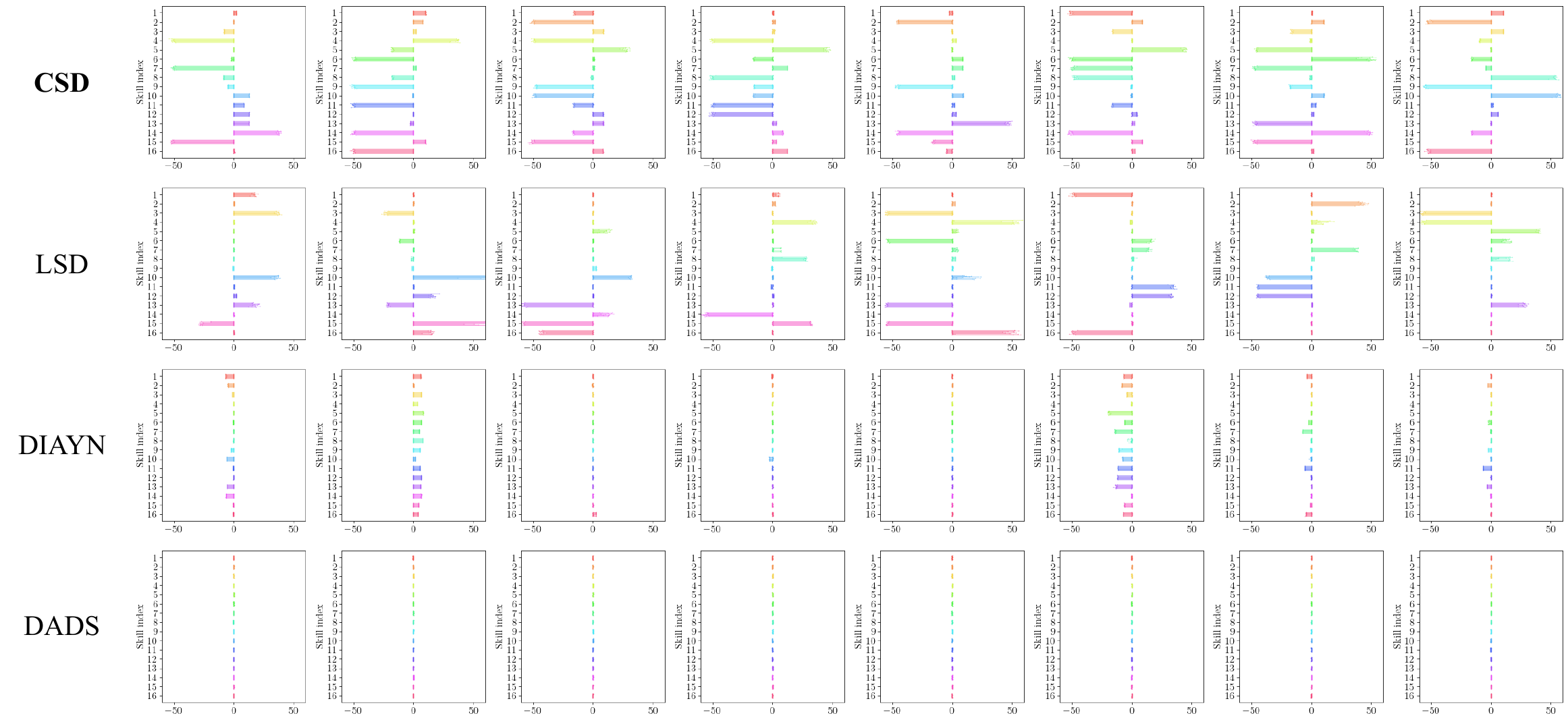}
        \caption{HalfCheetah $x$ trajectories}
    \end{subfigure}
    \caption{
    Complete qualitative results in Ant and HalfCheetah (eight runs for each method in each environment).
    We plot the skill trajectories of the agent with different colors.
    We note that in both environments, CSD not only learns locomotion skills but also discovers a variety of diverse skills,
    such as rotating and flipping.
    }
    \label{fig:qual_mujoco_all}
\end{figure*}

\clearpage

\section{Implementation Details}
\label{sec:appx_impl_details}

For manipulation environments,
we implement CSD on top of the publicly available codebase of MUSIC \citep{music_zhao2021}.
For MuJoCo environments, we implement CSD based on the publicly available codebase of LSD \cite{lsd_park2022}.
We mostly follow the hyperparameters used in the original implementations.
Our implementation can be found in the following repositories: {\url{https://github.com/seohongpark/CSD-manipulation}} (manipulation environments)
and {\url{https://github.com/seohongpark/CSD-locomotion}} (locomotion environments).
We run our experiments on an internal cluster with NVIDIA Tesla V100 and NVIDIA GeForce RTX 2080 Ti GPUs.
Each run mostly takes a day or less.

\subsection{Environments}
We adopt the same environment settings used in LSD \citep{lsd_park2022}
for Fetch manipulation environments (FetchPush, FetchSlide, FetchPickAndPlace) \citep{fetch_plappert2018}
and MuJoCo locomotion environments (Ant, HalfCheetah) \citep{mujoco_todorov2012,openaigym_brockman2016}.
In Fetch environments,
unlike LSD, we do not use any supervision, such as limiting the discriminator's input only to the object.
For the Kitchen environment, we use a $7$-DoF end-effector controller \citep{lexa_mendonca2021} with state-based observations.
We use an episode length of $200$ for locomotion environments and an episode length of $50$ for manipulation environments.
In locomotion environments, to ensure fair comparisons,
we use preset normalizers for all skill discovery methods as done in \citet{lsd_park2022},
but we find that CSD can still discover diverse behaviors including locomotion skills without a normalizer.

\subsection{Downstream Tasks}
\label{sec:appx_down}
\paragraph{Fetch environments.}
We use the same downstream tasks in \citet{lsd_park2022} for Fetch environments.
In FetchPushGoal, FetchSlideGoal, and FetchPickAndPlaceGoal, a goal position is randomly sampled at the beginning of each episode.
If the agent successfully places the object to the target position, a reward of $1$ is given to the agent and the episode ends.
We follow the original goal sampling range and reach criterion from \citet{fetch_plappert2018}.

\paragraph{Kitchen environment.}
We consider the following $13$ downstream tasks for the Kitchen environment:
BottomLeftBurner, BottomRightBurner, HingeCabinet, KettleBottomRight, KettleFall, KettleLift,
KettleTopLeft, KettleTopRight, LightSwitch, Microwave, SlideCabinet, TopLeftBurner, TopRightBurner.
For the success criteria of the tasks, we mostly follow \citet{kitchen_gupta2019,lexa_mendonca2021}
and refer to our implementation for detailed definitions.
As in the Fetch tasks, the agent gets a reward of $1$ when it satisfies the success criterion of each task.

\paragraph{MuJoCo locomotion environments.}
In AntGoal, a goal's $xy$ position is randomly sampled from $\text{Unif}([-20, 20]^2)$,
and if the agent reaches the goal, it gets a reward of $10$ and the episode ends.
In AntMultiGoals, the agent should follow four goals within $50$ steps each,
where goal positions are randomly sampled from $\text{Unif}([-7.5, 7.5]^2)$ centered at the current coordinates.
The agent gets a reward of $2.5$ every time it reaches a goal.
In HalfCheetahGoal, a goal's $x$ coordinate is randomly sampled from $\text{Unif}([-60, 60])$,
and if the agent reaches the goal, it gets a reward of $10$ and the episode ends.
For these three environments, we consider the agent to have reached the goal if it enters within a radius of $3$ from the goal.
In HalfCheetahHurdle, the agent gets a reward of $1$ if it jumps over a hurdle,
where we use the same hurdle positions from \citet{hurdle_qureshi2020}.

\subsection{Training.}

\paragraph{Skill policy.}
At the beginning of each episode,
we sample a skill $z$ from either a standard Gaussian distribution (for continuous skills)
or a uniform distribution (for discrete skills), and fix the skill throughout the episode.
For discrete skills,
we use standard one-hot vectors for DIAYN and DADS, and zero-centered one-hot vectors for CSD and LSD, following \citet{lsd_park2022}.
For DADS, we follow the original implementation choices,
such as the use of batch normalization and fixing the output variance of the skill dynamics model.
For CSD in manipulation environments, we start training the skill policy from epoch $4000$,
after the initial conditional density model has stabilized.
When modeling $\Sigma_\theta(s)$ of the conditional density model,
we use a diagonal covariance matrix
as we found it to be practically sufficient for our experiments.
Also, we normalize the diagonal elements with their geometric mean at each state for further stability.

We present the full list of the hyperparameters used in our experiments in \Cref{table:hyp_loc,table:hyp_man},
where we indicate the values considered for our hyperparameter search with curly brackets.
For the intrinsic reward coefficient, we use $50$ (DADS), $500$ (CSD and LSD), $1500$ (DIAYN),
$200$ (Disagreement Fetch), or $50$ (Disagreement Kitchen).
For the learning rate, we use $0.0001$ for all experiments except for CSD Humanoid, for which we find $0.0003$ to work better
(we also test $0.0003$ for LSD Humanoid for a fair comparison, but we find the default value of $0.0001$ to work better for LSD).
For the reward scale, we use $1$ (LSD, DIAYN, and DADS) or $10$ (CSD).
For the SAC $\alpha$, we use $0.003$ (LSD Ant and LSD HalfCheetah), $0.03$ (CSD Ant and LSD Humanoid), $0.1$ (CSD HalfCheetah), $0.3$ (CSD Humanoid), or auto-adjust (DIAYN and DADS).

\paragraph{High-level controller.}
After unsupervised skill discovery,
we train a high-level controller $\pi^h(z|s, g)$ that selects skills in a sequential manner for solving downstream tasks.
We use SAC \citep{sac_haarnoja2018} for continuous skills and PPO \citep{ppo_schulman2017} for discrete skills.
The high-level policy selects a new skill every $R$ steps.
We mostly follow the hyperparameters for low-level skill policies
and present the specific hyperparameters used for high-level controllers in
\Cref{table:hyp_man_high,table:hyp_loc_high}.

\begin{table}[t]
    \caption{Hyperparameters for manipulation environments.}
    \label{table:hyp_man}
    \vskip 0.15in
    \begin{center}
    \begin{tabular}{lc}
        \toprule
        Hyperparameter & Value \\
        \midrule
        Optimizer & Adam \citep{adam_kingma2014} \\
        Learning rate & $10^{-3}$ \\
        \# training epochs & $40000$ (Fetch), $20000$ (Kitchen) \\
        \# episodes per epoch & $2$ \\
        \# gradient steps per episode & $10$ \\
        Episode length & $50$ \\
        Minibatch size & $256$ \\
        Discount factor $\gamma$ & $0.98$ \\
        Replay buffer size & $10^5$ \\
        \# hidden layers & $2$ \\
        \# hidden units per layer & $256$ \\
        Nonlinearity & ReLU \\
        Target network smoothing coefficient $\tau$ & $0.995$ \\
        Random action probability & $0.3$ \\
        Action noise scale & $0.2$ \\
        Entropy coefficient & $0.02$ \\
        Intrinsic reward coefficient & $\{5, 15, 50, 150, 500, 1500, 5000\}$ \\
        CSD $\epsilon$ & $10^{-6}$ \\
        CSD initial $\lambda$ & $3000$ \\
        Disagreement ensemble size & $5$ \\
        \bottomrule
    \end{tabular}
    \end{center}
    \vskip -0.1in
\end{table}

\begin{table}[t]
    \caption{Hyperparameters for locomotion environments.}
    \label{table:hyp_loc}
    \vskip 0.15in
    \begin{center}
    \begin{tabular}{lc}
        \toprule
        Hyperparameter & Value \\
        \midrule
        Optimizer & Adam \citep{adam_kingma2014} \\
        Learning rate & $\{0.0001, 0.0003\}$ \\
        \# training epochs & $20000$ \\
        \# episodes per epoch & $5$ (Humanoid), $10$ (others) \\
        \# gradient steps per epoch & $64$ (policy), $32$ (others) \\
        Episode length & $200$ \\
        Minibatch size & $1024$ \\
        Discount factor $\gamma$ & $0.99$ \\
        Replay buffer size & $1000000$ (Humanoid), $2000$ (others) \\
        \# hidden layers & $2$ \\
        \# hidden units per layer & $1024$ (Humanoid), $512$ (others) \\
        Nonlinearity & ReLU \\
        Target network smoothing coefficient $\tau$ & $0.995$ \\
        Target network update frequency & every gradient step\tablefootnote{
        The original LSD implementation updates the target network every epoch, not every gradient step,
        but we find the latter to be about $10 \times$ sample efficient
        in terms of the number of environment steps.
        } \\
        SAC $\alpha$ & $\{0.001, 0.003, 0.01, 0.03, 0.1, 0.3, \text{auto-adjust \citep{saces_haarnoja2018}}\}$ \\
        Reward scale & $\{1, 10\}$ \\
        CSD $\epsilon$ & $10^{-6}$ \\
        CSD initial $\lambda$ & $3000$ \\
        Disagreement ensemble size & $5$ \\
        \bottomrule
    \end{tabular}
    \end{center}
    \vskip -0.1in
\end{table}

\begin{table}[t]
    \caption{Hyperparameters for SAC downstream policies in manipulation environments.}
    \label{table:hyp_man_high}
    \vskip 0.15in
    \begin{center}
    \begin{tabular}{lc}
        \toprule
        Hyperparameter & Value \\
        \midrule
        \# training epochs & $4000$ (Fetch), $8000$ (Kitchen) \\
        \# episodes per epoch & $16$ (Fetch), $2$ (Kitchen) \\
        \# gradient steps per epoch & $4$ (Fetch), $10$ (Kitchen) \\
        Replay buffer size & $10^6$ \\
        Skill sample frequency $R$ & $10$ \\
        Skill range & $[-1.5, 1.5]^D$ \\
        \bottomrule
    \end{tabular}
    \end{center}
    \vskip -0.1in
\end{table}

\begin{table}[t]
    \caption{Hyperparameters for PPO downstream policies in locomotion environments.}
    \label{table:hyp_loc_high}
    \vskip 0.15in
    \begin{center}
    \begin{tabular}{lc}
        \toprule
        Hyperparameter & Value \\
        \midrule
        Learning rate & $3 \times 10^{-4}$ \\
        \# training epochs & $1000$ \\
        \# episodes per epoch & $64$ \\
        \# gradient steps per episode & $10$ \\
        Minibatch size & $256$ \\
        Entropy coefficient & $0.01$ \\
        Skill sample frequency $R$ & $25$ \\
        \bottomrule
    \end{tabular}
    \end{center}
    \vskip -0.1in
\end{table}

\end{document}